\documentclass[10pt]{article}
\usepackage[margin=1in]{geometry}

\usepackage{url}

\usepackage[utf8]{inputenc} 
\usepackage[T1]{fontenc}    
\usepackage[colorlinks=true,linkcolor=blue,citecolor=blue]{hyperref}   
\usepackage{url}            
\usepackage{booktabs}       
\usepackage{amsfonts}       
\usepackage{nicefrac}       
\usepackage{microtype}      
\usepackage{amsmath,amssymb, amsthm, mathtools, dsfont}
\usepackage{cleveref}       
\usepackage{graphicx}
\usepackage[numbers,sort&compress]{natbib}
\usepackage{enumitem}
\usepackage{comment}
\usepackage{bbm}

\usepackage{algorithm}
\usepackage[noend]{algpseudocode}
\usepackage{algorithmicx}

\usepackage[parfill]{parskip}

\usepackage[labelfont=bf]{caption}
\usepackage{subcaption}
\usepackage{wrapfig}

\makeatletter
\newtheorem*{rep@theorem}{\rep@title}
\newcommand{\newreptheorem}[2]{%
\newenvironment{rep#1}[1]{%
 \def\rep@title{#2 \ref{##1}}%
 \begin{rep@theorem}}%
 {\end{rep@theorem}}}
\makeatother

\theoremstyle{plain}
\newtheorem{theorem}{Theorem}
\newreptheorem{theorem}{Theorem}

\newtheorem{lemma}{Lemma} 
\theoremstyle{definition}
\newtheorem{assumption}{A\!\!}
\theoremstyle{remark}
\newtheorem{remark}{Remark} 
\usepackage{mathrsfs}  

\usepackage{definitions}

\newcommand{\ind}{\perp\!\!\!\!\perp} 
\newcommand{\notind}{\not\!\perp\!\!\!\perp} 

\newcommand{\bU}{\mathbf{U}}
\newcommand{\bW}{\mathbf{W}}

\newcommand{\bX}{\mathbf{X}}

\newcommand{\bY}{\mathbf{Y}}

\newcommand{\bUt}{\tilde{\bU}}
\newcommand{\bA}{\mathbf{A}}
\newcommand{\bB}{\mathbf{B}}
\newcommand{\bR}{\mathbf{R}}
\newcommand{\bM}{\mathbf{M}}
\newcommand{\bS}{\mathbf{S}}
\newcommand{\bo}{\mathbf{o}}
\newcommand{\bH}{\mathbf{H}}
\newcommand{\bE}{\mathbf{E}}
\newcommand{\bF}{\mathbf{F}}
\newcommand{\bN}{\mathbf{N}}
\newcommand{\bv}{\mathbf{v}}

\newcommand{\bphi}{\boldsymbol{\phi}}
\newcommand{\bpsi}{\boldsymbol{\psi}}
\newcommand{\Rphi}{\mathbf{R}_\phi}
\newcommand{\Rpsi}{\mathbf{R}_\psi}
\newcommand{\bbA}{A}
\newcommand{\bbB}{B}

\DeclareMathOperator{\Domain}{Dom}

\newcommand{\xnew}{x_{\text{new}}}

\title{Adapting to Latent Subgroup Shifts via Concepts and Proxies}

\author{
Ibrahim Alabdulmohsin$^1$, Nicole Chiou$^{2,}\footnotemark[1]$ \hspace{0.1mm}, Alexander D'Amour$^1$, \\ 
Arthur Gretton$^3$, Sanmi Koyejo$^{1,2}$, Matt J. Kusner$^{4,}\footnotemark[1]$ \hspace{0.1mm}, Stephen R. Pfohl$^1$, \\ 
Olawale Salaudeen$^{5,}\footnotemark[1]$ \hspace{0.1mm}, Jessica Schrouff$^{1,}\footnotemark[2]$ \hspace{0.1mm}, Katherine Tsai$^{5,}\footnotemark[1]$
}

\date{
$^1$Google Research \\
$^2$Stanford University \\
$^3$Gatsby Computational Neuroscience Unit \\
$^4$University College London \\
$^5$University of Illinois Urbana-Champaign \\
\footnotetext{Authors listed in alphabetical order}
\footnotetext{Correspondence to m.kusner@ucl.ac.uk or alexdamour@google.com}
\footnotetext[1]{Work completed while at Google Research}
\footnotetext[2]{Now at Deepmind}
}

\begin{document}
\maketitle

\begin{abstract}
We address the problem of unsupervised domain adaptation when the source domain differs from the target domain because of a shift in the distribution of a latent subgroup. When this subgroup confounds all observed data, neither covariate shift nor label shift assumptions apply. We show that the optimal target predictor can be non-parametrically identified with the help of concept and proxy variables available only in the source domain, and unlabeled data from the target. The identification results are constructive, immediately suggesting an algorithm for estimating the optimal predictor in the target. For continuous observations, when this algorithm becomes impractical, we propose a latent variable model specific to the data generation process at hand. We show how the approach degrades as the size of the shift changes, and verify that it outperforms both covariate and label shift adjustment.

\end{abstract}

\section{Introduction}
Distribution shift is a fact of many real-world machine learning systems. For example, imagine we have trained a prediction model on patients of hospital $P$ and would like to apply it to patients of hospital $Q$. However, these hospitals differ in their patient populations along socioeconomic, demographic, and other axes \citep{Finlayson2021-zo}. How can we find the optimal predictor for hospital $Q$, given only labelled data from hospital $P$ and unlabelled data from hospital $Q$? This is the problem of unsupervised domain adaptation \citep{huang2006correcting}. 
Without any assumptions on the shift, this question is impossible to answer: the mapping from features $X$ to labels $Y$ could differ across hospitals in arbitrary ways. 
To address this, approaches typically assume that certain observed distributions are preserved across the shift, \emph{covariate shift}: $p(Y \mid X)\!=\!q(Y \mid X)$ \citep{shimodaira2000improving} or, \emph{label shift}: $p(X \mid Y)\!=\!q(X \mid Y)$ \citep{gart1966comparison}, where $p,q$ are distributions of hospitals $P,Q$.

However, these assumptions are often restrictive for real-world settings, as the shifts encountered are typically more complex (e.g., `compound' shifts \citep{schrouff2022maintaining}).
Here, we focus on one such shift that we call \emph{latent subgroup shift}.
Subgroup shift occurs when both the source $P$ and target $Q$ distributions are composed of a common set of subgroups $U \in \mathcal U$, but the prevalence of these subgroups differs, i.e., $p(U) \neq q(U)$.
The subgroup shift is latent if these subgroups are unobserved in both $P$ and $Q$.
Importantly, the relationships between features $X$ and labels $Y$ can differ between subgroups, such that neither the discriminative distribution $p(Y \mid X)$ nor the generative distribution $p(X \mid Y)$ is preserved across the shift.
In healthcare settings, these subgroups may differ in their exposure to social determinants of health, contributing to differences in health outcomes and patterns of comorbidity, care access, delivery, and treatment \citep{marmot2005social}.

\begin{figure*}[t]
    \centering
    \includegraphics[width=0.8\textwidth]{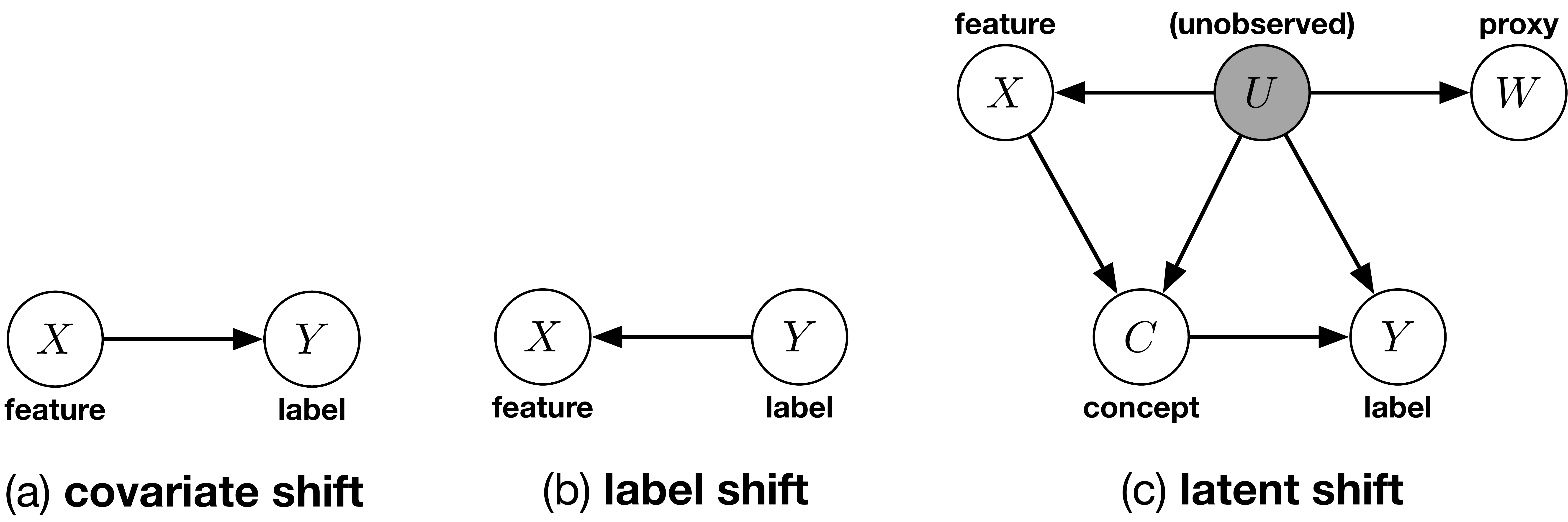}
    \vspace{-1ex}
    \caption{Different domain shift assumptions: (a) $p(X) \!\neq\! q(X)$, (b)  $p(Y) \!\neq\! q(Y)$, (c)  $p(U) \!\neq\! q(U)$.}
    \label{fig:model}
    \vspace{-2ex}
\end{figure*}

To tackle latent subgroup shift, 
we frame learning the optimal $q(Y \mid X)$ as an identification problem.
Our identification strategy combines approaches from proximal causal inference (originally designed to identify intervention distributions $p(Y \mid do(X))$ under unobserved confounding using proxy variables) \citep{kuroki2014measurement}, black box label shift adaptation \citep{lipton2018detecting}, and concept bottleneck modeling \citep{koh2020concept}. We show that it is possible to express $q(Y \mid X)$ in terms of the joint distribution of observables in $P$ and the distribution of unlabeled inputs $X$ in $Q$.
We derive two identification results, one for discrete data and another for continuous data. While the results are constructive, immediately implying an algorithm, estimation requires non-trivial density estimation. Therefore, we describe an alternative approach that leverages stable latent variable \citep{kingma2013auto} models to estimate $q(Y|X)$. 
The proposed approach answers an open question on how to leverage advances in concept bottleneck models \citep{koh2020concept} for  distribution shifts in both $X$ and $Y$. Further, it allows one to learn a single model in the source $P$ which can then be adapted to arbitrary shifts in $Q$. 

{\bf Contributions.} We propose a new approach for adaptation to latent distribution shifts given concepts and proxies, for cases where existing adaptation methods often fail. We formally identify the target distribution for both discrete and continuous variables, then propose effective estimators. We perform a sensitivity analysis that characterizes how our method changes when shift size and proxy strength are varied. We show that our approach 
outperforms multiple baselines including covariate and label shift techniques.

{\bf Notation. } We denote scalars and functions by lowercase letters (e.g., $a$, $a(\cdot)$), vectors by bold lowercase ($\ab$), random variables by capital letters ($A$), matrices by bold, capital letters ($\Ab$), and sets by caligraphic, capital letters ($\mathcal{A}$).  Let $[n]$ denote the set $\{1, 2,  \ldots, n\}$.
\section{Related Work}

There has been a flurry of recent work on improving out-of-distribution generalization (see \citet{shen2021towards}, \citet{wang2022generalizing}, and \citet{zhou2022domain} for three recent surveys). Largely, this work can be divided into two camps: (a) work that learns a single model to work well across shifts, such as work on invariant predictors \citep{arjovsky2019invariant} and, (b) work that adapts a model from a source distribution to a target distribution, given access to limited data in the target. Here we focus on the second class of approaches. 
\paragraph{Model adaptation.}
To obtain an optimal predictor in a new distribution $Q$, one of the most popular assumptions is to localize the shift between distributions $P$ and $Q$ in the features (covariates) $X$, i.e., \emph{covariate shift}: $p(X) \!\neq\! q(X)$.  
There has been a large body of work devoted to estimating predictors for $Q$ under this setting \citep{shimodaira2000improving,zadrozny2004learning,huang2006correcting,gretton2009covariate,bickel2009discriminative,sugiyama2012machine,chen2016robust,schneider2020improving}. The key assumption in this line of work is that $p(Y \mid X) \!=\! q(Y \mid X)$. Therefore, if one makes the source data appear like the target data (e.g., by reweighing the source classifier loss by $q(X)/p(X)$), one can learn an accurate target classifier. 
The other popular assumption is to localize the shift in the labels $Y$, i.e., \emph{label shift}: $p(Y) \!\neq\! q(Y)$ and $p(X \mid Y) \!=\! q(X \mid Y)$ \citep{gart1966comparison,manski1977estimation,rosenbaum1983central,saerens2002adjusting,forman2008quantifying,storkey2009training,du2012semi,zhang2013domain,lipton2018detecting,azizzadenesheli2019regularized,alexandari2020maximum,garg2020unified,tachet2020domain,wu2021online}. Here one can use a similar approach: learn $q(Y)/p(Y)$ and use it to reweigh a source classifier, adapting it to the target distribution. The assumptions of covariate and label shift can be framed as criteria on the causal structure of the data, shown in Figure~\ref{fig:model}(a)-(b) \citep{scholkopf2012on}. Most theoretical work is on generalization error bounds for covariate shift \citep{sugiyama2005generalization,ben2006analysis,mansour2009domain,ben2010theory,cortes2011domain,johansson2019support} and label shift \citep{gong2016domain}. 

\paragraph{Causality for domain shift.} 
Recently, a line of work has framed domain shift using causal methods \citep{zhang2015multi,Magliacane2018-oh,gong2018causal,chen2020domain,teshima2020few}. 
Most related to our approach is the work of \citet{yue2021transporting}. 
Similar to our setup, they describe a setting where an unobserved latent confounder $U$ shifts the distribution of $X$ and $Y$. 
However, different from our work, they target an interventional distribution instead of $q(Y\mid X)$. To do so they learn mappings from $X \sim P$ to $X \sim Q$, and vice-versa. They use these mappings, as well as a variational autoencoder \citep{kingma2013auto}, to generate two `proxies', one for $X$ and $Y$. They assume these proxies are caused by $U$, and they use the result of \citet{miao2018identifying} to identify an invariant `bridge function' to remove the effect of the latent shift. However, this does not guarantee identification of the structural equations mapping $U$ to the proxies, $X$, and $Y$, which is necessary for the procedure to correct for $U$.
\section{Setup and Preliminaries}
\label{sec:setup}

Let $P$ be the source distribution and $Q$ be the target, with probability mass/density functions $p$ and $q$. 
Our goal is to identify the optimal predictor of $Y$ from $X$ in the target: $q(Y|X)$.
To do so, we will make two main assumptions.
First, to make progress in this setting, we assume that we have access to some auxiliary variables that play key roles in the source distribution.

\begin{assumption}\label{asm:graph}
\emph{We also observe auxiliary variables $C$ (concept bottleneck) and $W$ (proxy). All data is generated by the process described in Figure~\ref{fig:model}(c) and is \emph{faithful} and \emph{Markov} \citep{spirtes2000causation} (i.e., conditional independences in the data exist iff they exist in the graph). Crucially, we only observe $(X,C,Y,W)$ in the source $P$ and $X$ in the target $Q$.}
\end{assumption}
\vspace{-1ex}

Formally, the data generation process of Figure~\ref{fig:model}(c) is a probabilistic graphical model  \citep{pearl1988probabilistic}. Given a set of observed variables $\mathcal{V}$ and unobserved variables $\mathcal{U}$, these models define a  functional relationship $f_i$ between each $V_i \!\in'! \mathcal{V}$ and the variables that generate $V_i$ (also called its direct \emph{parents}) $\mathcal{V}_{\mathrm{pa}(i)},\mathcal{U}_{\mathrm{pa}(i)}$ i.e., $V_i = f_i(\mathcal{V}_{\mathrm{pa}(i)}, \mathcal{U}_{\mathrm{pa}(i)})$. These relationships can be described by a directed acyclic graph (DAG), e.g., as in Figure~\ref{fig:model}. A key aspect of these models is that they encode conditional independence relationships between variables in $\mathcal{V}, \mathcal{U}$, that can be derived via d-separation \citep{pearl2000models}. Throughout this work we assume $f_i$ and $\mathcal{U}$ are unknown.

The additional auxiliary variables $C$ and $W$ play specific roles in this graph.
$C$ operates as a ``concept bottleneck'' that mediates the dependence between $X$ and $Y$ within subgroups indexed by $U$.
Meanwhile, $W$ operates as an independent proxy, or noisy observation, of $U$ that is conditionally independent of all other variables.
Both of these properties play key roles in our identification strategy.

Our second assumption defines latent subgroup shift. 

\begin{assumption}\label{asm:coufounder_shift}
\emph{The shift between $P$ and $Q$ is located in $U$, i.e., there is a} latent shift $p(U) \neq q(U)$, \emph{while $p(\mathcal{V}|U) \!=\! q(\mathcal{V}|U)$, where
$\mathcal{V} \subseteq \{W,X,C,Y\}$.}
\end{assumption}

Under these assumptions, distributions on $U$ or that have $U$ marginalized (i.e., all observed distributions $p(\mathcal{V}) \neq q(\mathcal{V})$ for $\mathcal{V} \subseteq \{W,X,C,Y\}$) will shift between $P$ and $Q$, whereas only distributions conditional on $U$ do not shift. This is a direct generalization of the covariate shift invariance, in which $U\to X\to Y$ and the label shift invariances, in which $U\to Y\to X$. 

Our framework is inspired by (a) concept bottleneck models \citep{koh2020concept} and (b) identification via proxies \citep{kuroki2014measurement}. We briefly review these topics next.

\paragraph{Concept bottleneck models.}
Data in certain settings may contain information beyond features and labels. For instance, in healthcare it is common to not only have raw electronic health record data $X$ (e.g., temperature, blood cultures, ...) and disease labels $Y$, but also physician summaries $C$ such as the presence and spread of infection. \citet{koh2020concept} formalize this learning setup, calling $C$ \emph{concepts}. In general, concepts $C$ are high-level, often interpretable, pieces of information that mediate the relationship between $X$ and $Y$. Prior works have used concepts for diagnosing model failures and for covariate shift 
\citep{kumar2009attribute,lampert2009learning,koh2020concept,Chen2020-dt,Mahinpei2021-ms}. The concept bottleneck model \citep{koh2020concept} was shown to be robust to covariate distribution shifts; here, we show with the appropriate adjustment strategy, such models can also be adapted to subgroup shifts. Another line of work have incorporated concepts into causal models to improve model explanations \citep{goyal2019explaining,bahadori2021debiasing}.

{\bf Proxies.} Our work leverages results in causal effect estimation with proxy variables \citep{kuroki2014measurement,miao2018identifying}. In these works, $W$ is a proxy of $U$ that allows one to identify the causal effect of $C$ on $Y$ in Figure~\ref{fig:model}(c). In our running example, a useful $W$ would be the region where a patient lives as this is often a proxy for SDH quantities, such as income $U$.

\section{Identification Under Latent Shift}
\label{sec:identification}
In this section, we report identification results for the optimal target distribution predictor $q(Y \mid X)$ given observed draws from $p(X, C, Y, W)$ and $q(X)$.
We first present our central adjustment strategy in the case where $U$ is observed in the source distribution.
We then show that, when $C$ and $W$ are observed in the source distribution, we can use this strategy even in cases where $U$ is unobserved.
We consider two such cases: one where all observed variables are discrete, and another where $X$ and $W$ are continuous.
In these latter two cases, the key challenge is to show that the distributions in our adjustment formula, which involve $U$, can be identified in the source domain.

\subsection{Subgroup Adjustment Formula}
To begin, we present our central adjustment formula, considering the case where $U$ is observed in the source distribution $P$, but not in the target distribution $Q$.
We derive the formula by decomposing our target $q(Y \mid X)$, leveraging A\ref{asm:coufounder_shift} and Figure~\ref{fig:model}(c):
\begin{align}
     q(Y|X) &\; \stackrel{(a)}{=} \sum_{i=1}^{k_U} q(Y|X,U=i)q(U=i|X)   \nonumber \\
    &\; \stackrel{(b)}{=} \sum_{i=1}^{k_U} p(Y|X,U=i)\frac{q(X|U=i)q(U=i)}{q(X)} \nonumber \\
    &\; \stackrel{(c)}{=} \sum_{i=1}^{k_U} p(Y|X,U=i)\frac{p(U=i|X)p(X)q(U=i)}{p(U=i)q(X)} \nonumber \\
    &\; \stackrel{(d)}{\propto} \sum_{i=1}^{k_U} p(Y|X,U=i)p(U=i|X)\frac{q(U=i)}{p(U=i)} \label{eq:goal}
\end{align}
The first equality $(a)$ is given by the chain rule and marginalization. The second $(b)$ is given by A\ref{asm:coufounder_shift}: since $q(Y|X,U\!=\!i)$ conditions on $U$, we have $q(Y|X,U\!=\!i) \!=\! p(Y|X,U\!=\!i)$. The fractional term is given by Bayes rule. The equality $(c)$ is again given by A\ref{asm:coufounder_shift} and Bayes rule: $q(X|U\!=\!i) \!=\! p(X|U\!=\!i) \!=\! p(U\!=\!i|X)p(X)/p(U\!=\!i)$. The proportional $(d)$ is given by the fact that $p(X)/q(X)$ is constant as the left-hand side conditions on these variables.

When $U$ is observed under $P$, all quantities on the final right hand side are directly estimable except $q(U)/p(U)$, because $U$ is not observed under $Q$. Interestingly, this parallels the label shift problem, where distributions conditional on $Y$ are preserved across the distribution shift, but $Y$ is not observed under $Q$.
In fact, the same label shift adaptation identification arguments and techniques can be applied to adjust for $U$ instead!
Here, we adapt the method-of-moments identification argument made in \citet{lipton2018detecting}.
For any function $f(X)$, the identity $q(f(X)) = \sum_{i=1}^{k_U} q(f(X) \mid U=i) q(U)$ can be expanded (using Bayes rule and A~\ref{asm:coufounder_shift}):
\begin{align}
\frac{q(f(X))}{p(f(X))} = \sum_{i=1}^{k_U} p(U=i \mid f(X)) \frac{q(U=i)}{p(U=i)}.\label{eq:label shift LR}
\end{align}
These equations define a linear system, and, for appropriate choices of $f(X)$ and rank conditions on $p(U = i \mid X)$ (see A\ref{asm:inversion} and A\ref{asm:mat_invert} below), we can solve for $q(U=i)/p(U=i)$.
For example, \citet{lipton2018detecting} define $f(X)$ as the decision function of a classifier; in that case the linear system can be written in terms of the confusion matrix of the classifier.
\citet{garg2020unified} discuss other choices, as well as maximum likelihood approaches to learning this likelihood ratio.
Upon solving \eqref{eq:label shift LR}, $q(Y \mid X)$ is identified by \eqref{eq:goal}.

\begin{remark}
This ``observed $U$'' setting is a simplification of the general latent subgroup shift problem, but may be of independent interest. In many applications, especially when $U$ includes sensitive demographic categories, the subgroup label may be collected at training time, but unavailable at deployment time. In such cases, this identification argument would be sufficient for domain adaptation. 
\end{remark} 

\begin{remark}
The identifying expression \eqref{eq:goal} enables adaptation to new distributions $Q$ without retraining any models under $P$. To adapt to a new distribution, we plug in a new estimate of $q(f(X))$ to \eqref{eq:label shift LR}, then evaluation \eqref{eq:goal} at the solution. This \emph{post hoc} property applies to all identification strategies we discuss.
\end{remark}

\subsection{The Error of Covariate/Label Adjustment}

What if we apply covariate or label shift adjustment to the latent subgroup shift setting?

\paragraph{Covariate shift adjustment.}
Assume data follows the latent shift setting of Figure~\ref{fig:model}(c), but we (falsely) believe that the shift between the observed data in $P$, $\{X,C,W,Y\}$, and that of $Q$, $\{X’\}$, is due to covariate shift. The covariate shift assumption implies that $P(Y|X) \!=\! Q(Y|X)$. Given this, we  would start by training a model $f: X \!\rightarrow\! Y$ on the data in $P$ which estimates $P(Y|X)$. We would then use this model on the data $X’$ in $Q$ as an estimate $Q(Y|X)$ (we would only use $X$ to train $f$, and not $(C,W)$, as we only see $X’$ in $Q$). 
However, regardless of the amount of data in $P$ and $X$ there would always be an error between $f(X) := P(Y|X)$ and $Q(Y|X)$. Specifically, at the population level, the (squared) error under latent shift is: 
\begin{align*}
   ( P(Y|X) &- Q(Y|X) )^2 = \Big( \sum_u P(Y|X,u) \big[ P(u|X) - Q(u|X)\big] \Big)^2 \\
   &= \Bigg( \sum_u P(Y|X,u)P(u|X) \Bigg[ 1 -  \frac{P(X)}{Q(X)} \frac{Q(U)}{P(U)} \Bigg] \Bigg)^2, 
\end{align*}
which is non-zero so long as both $P(X) \neq Q(X)$ and $P(U) \neq Q(U)$. This only happens if there is no shift and is easy to verify.

\paragraph{Label shift adjustment.}
Imagine we instead assumed the shift was due to label shift which implies $P(X|Y) = Q(X|Y)$. Given this, $Q(Y|X)$ could be written as: 
\begin{align*}
Q(Y|X) =&\; Q(X|Y)\frac{Q(Y)}{Q(X)} = P(X|Y)\frac{Q(Y)}{Q(X)} \\
=&\; P(Y|X)  \frac{P(X)}{Q(X)}  \frac{Q(Y)}{P(Y)}.
\end{align*}
All of the terms on the right hand side are  estimable, even $Q(Y)/P(Y)$. Specifically, given a trained model $f: X \!\rightarrow\! Y$ on the data in $P$ (estimating $P(Y|X)$), we can estimate $Q(Y) / P(Y)$ using a label shift correction technique. For example, \citet{lipton2018detecting} shows that $Q(f(X)) \!=\! \sum_y P(f(X), y) [ Q(y) / P(y)]$. However, this adjusted estimate also incurs error with respect to the optimal target $Q(Y|X)$ in the latent shift setting. The population (squared) error under latent shift is: 
\begin{align*}
 \Big(P&(Y \mid X) \frac{P(X)}{Q(X)}\frac{Q(Y)}{P(Y)} - Q(Y \mid X)\Big)^2 \\
 =& \Bigg(\sum_u P(Y \mid X, u)P(u \mid X) \frac{P(X)}{Q(X)} \Bigg[\frac{Q(Y)}{P(Y)} - \frac{Q(u)}{P(u)}\Bigg]\Bigg)^2
\end{align*}
which is non-zero so long as $Q(Y) / P(Y) \neq Q(u) / P(u)$ for all $u \in U$. This can only happen if $Q(Y)/P(Y) = 1$, i.e., there is no shift. We can prove this by contradiction. Assume either (a) $Q(Y)/P(Y) > 1$ or (b) $Q(Y)/P(Y) < 1$. In case (a), we have that $Q(u) > P(u)$ for all $u \in U$, and in case (b) we have that $Q(u) < P(u)$. However, neither can be the case as both $Q(U)$ and $P(U)$ must sum to $1$. Therefore, we must have that $Q(Y)/P(Y) = 1$, the no-shift setting, which is again easy to verify.

\subsection{Discrete Observations}

We now state sufficient conditions for identification of $q(Y|X)$ in the latent shift setting.
To begin we assume all observable variables $\{X, C, Y, W\}$ are discrete.

\begin{assumption}\label{asm:k_U}
\emph{$U \in [k_U]$ is discrete s.t., $k_X,k_W \geq k_U$ (recall $k_X,k_W$ are the number of categories of (discrete) $X,W$).}
\end{assumption}
\vspace{-1ex}

Generally, identification requires some restrictions on how $U$ influences the observed variables $\{W, X, C, Y\}$.
The above places such a restriction more generically than restriction functional forms;
all we require is that the support of $U$ is smaller than that of observed variables $X,W$.

\begin{assumption}\label{asm:inversion}
\emph{For every $i \in [k_U]$ where $q(U=i) > 0$ we have  $p(U=i) > 0$, all linear systems have rank at least $k_U$, and $p(Y|C,U=1) \neq p(Y|C,U=2) \neq \cdots \neq p(Y|C,U=k_U)$ $P$-almost everywhere.}
\end{assumption}

The first condition ensures that $q(U=i)/p(U=i)$ is bounded for all $i$. The remaining two conditions are inherited from \citet{kuroki2014measurement}: they ensure that inverses exist and that eigenvectors are unique.
Essentially, they require that all variables depend non-trivially on $U$.
Overall these assumptions are of two types: (1) \textbf{Structural}: A\ref{asm:graph} and A\ref{asm:coufounder_shift} describe how the data and shifts are structured; (2) \textbf{Functional}: S\ref{asm:k_U} and A\ref{asm:inversion} detail conditions on the functions that generate data.

Our main result for discrete data is the following.

\begin{lemma}\label{lem:u}
Given A\ref{asm:graph}--A\ref{asm:inversion}, all probability mass functions over discrete $\{W,X,C,Y,\tilde{U}\}$ in the source $P$ are identifiable, where $\tilde{U}$ is an unknown permutation of $U$.
\end{lemma}

\begin{theorem}[Identifiability for Discrete Observations]\label{thm:disc_q_y_x}
The distribution $q(Y|X)$ is identifiable from discrete $\{W,X,C,Y,\tilde{U}\} \sim P$ and $X \sim Q$.
\end{theorem}

\paragraph{Proof sketches.} We give full proofs in the Appendix and give sketches here. The first key observation for Theorem~\ref{thm:disc_q_y_x} is that all of the steps $(a)$--$(d)$ in eq.~(\ref{eq:goal}) hold when $U$ is replaced with the permutation $\tilde{U}$. This is because (a) $\tilde{U}$ satisfies the same independence conditions as $U$, and (b) $q(Y|X)$ only requires marginalizing over $U$, making the order of the categories of $U$ irrelevant to identification.
Given Lemma~\ref{lem:u}, the only step remaining is to solve \eqref{eq:label shift LR} in terms of $\tilde U$.
A\ref{asm:k_U} and A\ref{asm:inversion} ensure that the system has a solution.

The proof of Lemma~\ref{lem:u} works in two stages: 1. It first demonstrates that $p(W|\tilde{U})$ can be identified, and 2. It shows that once $p(W|\tilde{U})$ is identified, all distributions on $W,X,C,Y,\tilde{U}$ are identified. Stage 1 is done by proving a variation of a result given by \citet{kuroki2014measurement}. They demonstrate that when $k_W\!=\!k_X\!=\!k_U$ and data is generated from the graph of Figure~\ref{fig:model}(c) then it is possible to identify the causal effect $p(Y|do(C))$ (in Theorem 1 \citep{kuroki2014measurement}).
Identifying $p(Y|do(C))$ only requires identifying specific distributions involving $\tilde{U}$, in order to remove its contribution to $Y$, i.e., $p(Y|do(C)) = \sum_{x,u} P(Y|C,X\!=\!x,\tilde{U}\!=\!u)P(X=x,\tilde{U}\!=\!u)$.
However, as we show by construction, the result of \citet{kuroki2014measurement} is stronger. In Stage 1, we recover $p(W|\tilde{U})$ for Figure~\ref{fig:model} (c) by contrasting the distributions $p(X, W \mid c)$ and $p(y, X, W \mid c)$.
Specifically, $p(W|\tilde U)$ can be recovered from the eigendecomposition of $\bA^{-1}\bB$ where, for fixed values of $y$ and $c$, these matrices are as follows,
\begin{equation}
    \resizebox{0.94\hsize}{!}{
    $\underbrace{\begin{bmatrix}
        1 & p(w_1 | c) & \cdots & p(w_{k_W-1} |c) \\
        p(x_1 | c) & p(x_1, w_1 | c) &  \cdots & p(x_1, w_{k_W-1} | c) \\
        \vdots & \vdots & \ddots & \vdots \\
        p(x_{k_X-1} | c) & p(x_{k_X-1}, w_1 | c), & \cdots & p(x_{k_X-1}, w_{k_W-1} | c)
    \end{bmatrix}}_{\bA}$
    \hspace{3mm}
    $\underbrace{\begin{bmatrix}
        p(y | c) & p(y, w_1 | c) & \cdots & p(y, w_{k_W-1} | c) \\
        p(y, x_1 | c) & p(y, x_1, w_1 | c) & \cdots & p(y, x_1, w_{k_W-1} | c) \\
        \vdots & \vdots & \ddots & \vdots \\
        p(y, x_{k_X-1} | c) & p(y, x_{k_X-1}, w_1 | c) & \cdots & p(y, x_{k_X-1}, w_{k_W-1} | c)
    \end{bmatrix}}_{\bB}$.    \label{eq:AinvB}
    }
    \nonumber
\end{equation}
In the above $w_1$ is shorthand for $W=1$ (similarly for $X$). In Stage 2, we identify all distributions involving $\tilde{U}$. The key observation behind this second result is that conditioning on $\tilde{U}$ d-separates $W$ from the rest of the observed variables. Thus, factorizing observed distributions using $\tilde{U},W$ can form linear systems. In these systems, the unknown distributions involving $\tilde{U}$ can be recovered by some function of $p(W|\tilde{U})$ (identified in Stage 1) and observables.

\paragraph{Estimation.} As both proofs are constructive, we can immediately use them to design an approach to estimate $q(Y|X)$. This is shown in Algorithm~\ref{algo:method}.

\begin{algorithm}[t]
\caption{Estimating $q(Y|X)$.}\label{algo:method}
\begin{algorithmic}[1]
  \Require
  source $\mathcal{P} \!=\! \{(w_i, x_i, c_i, y_i)\}_{i=1}^n$;
  target $\mathcal{Q} \!=\! \{x_j\}_{j=1}^m$;
  For any variables $G \!\in\! [k_G], H \!\in\! [k_H]$ let $p(\mathbf{G}|\mathbf{H})$ be a $k_G \times k_H$ matrix of probabilities s.t. $p(\mathbf{G}|\mathbf{H})_{ij} = p(G\!=\!i|H\!=\!j)$
  \State Using $\mathcal{P}$, form matrices $\mathbf{A},\mathbf{B}$ described in eq.~(\ref{eq:AinvB})
  \State Decompose $\mathbf{A}^{-1}\mathbf{B} \!=\! \mathbf{S}^{-1}\Lambda\mathbf{S}$ to get $p(\mathbf{W}|\tilde{\mathbf{U}})$ from $\mathbf{S}^{-1}$
  \State Compute $p(\tilde{\mathbf{U}}|\mathbf{X}) = p(\mathbf{W}|\tilde{\mathbf{U}})^{-1}p(\mathbf{W}|\mathbf{X})$
  \State Compute $q(\bUt)/p(\bUt) = p(\bUt|\bX)^{-1}[q(\bX)/p(\bX)]$
  \State Compute $p(\bY| X, \bUt) = p(\bY | X, \bW)\Big(\frac{p(\bW | \bUt) \circ p(\bUt | X)}{p(\bW | X)}\Big)^{-1}$
  \State Compute $q(\bY|x_j) \!\propto\! p(\bY| x_j, \bUt)\big[ p(\tilde{\mathbf{U}}|x_j) \!\circ\! \frac{q(\bUt)}{p(\bUt)} \big], \forall x_j \!\in\! \mathcal{Q}$.
\end{algorithmic}
\end{algorithm}

\subsection{Continuous Observations}

We now consider the case where $W,X,C,Y$ are continuous. This setting turns out to be more challenging, as, unlike in the discrete case, we cannot enumerate all of the states and apply finite dimensional eigendecomposition to estimate the associated probability mass functions. Instead, we must apply functional analysis tools to estimate nonparametric continuous probability density functions, which require more care to ensure existence and estimability.
To this end, we make the following assumptions.

\begin{assumption}\label{asm:w_u_lindep}
There exists a $c\in\Domain(C)$, such that $p(X\mid U\!=\!i, c), p(X\mid U\!=\!j, c)$ are linearly independent for all $(i,j) \in[k_U]$ for $i \neq j$. Similarly, 
$p(W\mid U\!=\!i), p(W\mid U\!=\!j)$ are linearly independent for all $(i,j) \in[k_U]$ for $i \neq j$.
\end{assumption}

This assumption allows us to identify the distributions of $p(W\mid U)$ and $p(X\mid U,C)$, which are crucial to the eigendecomposition technique.

\begin{assumption}\label{asm:mat_invert}
There exist distinct points $x_1,\ldots,x_{k_U}\!\in\!\Domain(X)$ such that the matrix $[p(U\!=\!j\mid x_i)]_{i,j}\in\RR^{k_U\times k_U}$ is invertible.
\end{assumption}

This assumption ensures that the $q(U)/p(U)$ system in eq.~\eqref{eq:label shift LR} has a unique solution. Note this assumption is very weak for continuous $X$, e.g., $x_1,\ldots,x_{k_U}$ can be chosen to be exemplars of each class $i \in [k_U]$.

With A\ref{asm:w_u_lindep}, A\ref{asm:mat_invert} replacing A\ref{asm:k_U}, we extend the identification result from Theorem~\ref{thm:disc_q_y_x} to continuous data. 

\begin{theorem}[Continuous Observations]\label{thm:cont_q_y_x}
Given A\ref{asm:graph}, A\ref{asm:coufounder_shift}, A\ref{asm:inversion}--\ref{asm:mat_invert}, the distribution $q(Y|X=x)$ is identifiable from continuous $\{W,X,C,Y\} \sim P$ and $x \in X \sim Q$.
\end{theorem}

We give a full proof in the Appendix. The steps are similar to the discrete observation case: set up a linear system, eigendecompose it, recover $p(W|\tilde{U})$ from the eigenvectors, and use $p(W|\tilde{U})$ to identify all quantities on the right-hand side of eq.~(\ref{eq:goal}). However, the specifics of the continuous setting require more technical tools.

\paragraph{Estimation.} Implementing a plug-in estimator from Theorem~\ref{thm:cont_q_y_x} is challenging, as it requires non-parametric conditional density estimation and an eigendecomposition over functions.
We implement such an approach, and describe it in detail in the Appendix.

\section{Roles of Concepts and Proxies}
\label{sec:uniqueness}

Do we really need $C$ and $W$? And why can't we have additional edges in Figure~\ref{fig:model}(c), e.g. $X \rightarrow Y$? We describe here why the ``concept bottleneck'' and ``proxy'' properties of $W$ and $C$ are essential to our identification strategy. Specifically, we discuss at a high level why generalizing the graph by removing observed nodes or adding edges prevents non-parametric identification of simpler causal quantities. While these are not necessary conditions, they are nearly as general as those used in non-parametric identification results in causal inference literature (\citet{miao2018identifying,lee2021causal} also allow edge $W \rightarrow Y$).

\paragraph{Can $C$ and/or $W$ be removed?} Removing $C$ corresponds to the setting of \citet{pearl2010measurement}, where the goal is to estimate $p(Y|do(X))$. This work assumes one can either: (a) observe $U$ without error in a subpopulation \citep{selen1986adjusting,greenland2008bias}, (b) observe $p(W|U)$ \citep{pearl2010measurement}, or (c) place a prior distribution on the parameters of $p(W|U)$ to bound $p(Y|do(X))$ \citep{greenland2005multiple}. However, these techniques are non-trivial when $U$ is complex. Here we will not assume that it is possible to observe $U, p(W|U)$ or derive a prior for $p(W|U)$. Keeping $C$ but removing $W$ leads to a generalization of the front-door graph \citep{pearl2000models} for which causal effects are not non-parametrically identifiable. If we remove both $C$ and $W$, we can only identify $p(Y|do(X))$ if $U$ is observed, an assumption called `ignorability' \citep{imbens2015causal}.

\begin{figure}[t!]
\centering
\includegraphics[width=0.3\textwidth]{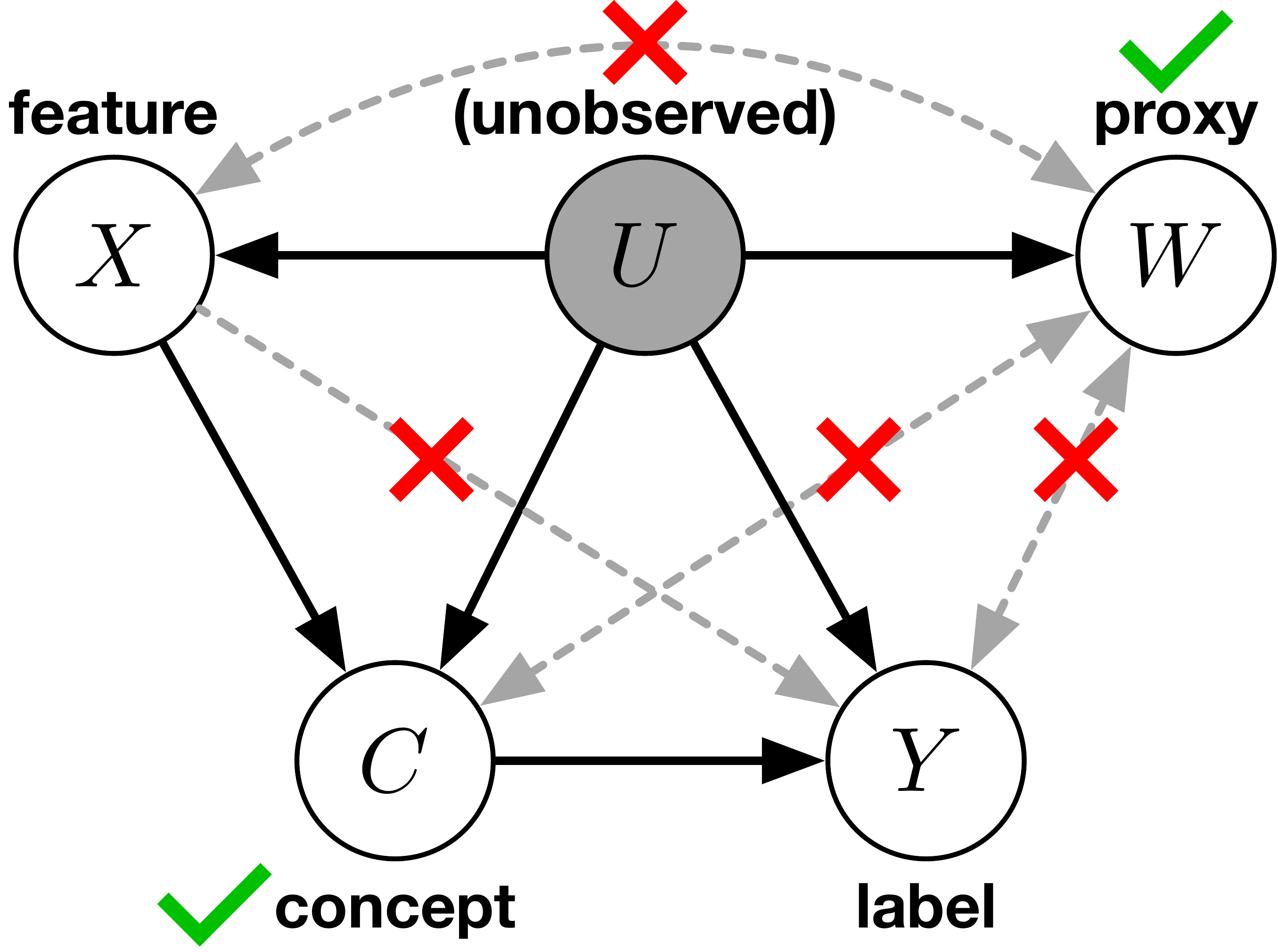}
\caption{Removing $C,W$ or adding any of the dotted edges prevents non-parametric 
identification of the full joint distribution $p(\mathcal V, \tilde{U})$ via our approach.}
\label{fig:non-graphs}
\end{figure}

\paragraph{Can we remove/add any additional edges?}
First note that if we remove edges from our assumed graph this limits the possible data distributions that it could have generated. This is because when edges are removed, conditional independences may be introduced. For example, if we remove the edge from $U \rightarrow C$ then $W \ind C \mid X$, which is not the case for our original graph in Figure~\ref{fig:model} (c). Another way to see this is that we can recover the covariate shift graph of Figure~\ref{fig:model} (a) from ours if we remove all edges starting from $U$, then remove $X \rightarrow C$, and finally relabel $C$ as $X$. Recall that the covariate shift graph implies $p(Y|X)\!=\!q(Y|X)$ which does not hold in our original graph. What about adding edges? Identifying $p(W|U)$ (i.e., Stage 1 in the proof of Lemma~\ref{lem:u}) requires that both $W \ind \{X,C,Y\} \mid U$ and $Y \ind \{W,X\} \mid \{U, C\}$. The first conditional independence is broken if there are any arrows from $X,C,Y$ to or from $W$. We do not prove here that this is necessary, but we suspect that it is: currently the only edge that can be added for identifying the simpler causal quantity $p(Y|do(C))$ is $W \rightarrow Y$ \citep{miao2018identifying,lee2021causal}. Crucially, these methods do not identify $p(W|U)$. The only other edge that could be added to the graph and it still be a DAG is $X \rightarrow Y$. However, this would break the second conditional independence statement as it would make $Y \notind X \mid \{U, C\}$. This edge would also render the causal effect unidentifiable under the most generic non-parametric methods \citep{lee2021causal}.

\section{Estimation with Latent Variable Models}\label{sect:estimation}

Algorithm~\ref{algo:method} and its associated continuous version (described in the Appendix~\ref{sec:appendix:continuous}) become impractical as the dimension increases (due to the need for probability mass/density estimation). Here, we propose an alternative approach based in deep latent variable modelling that can be useful for adapting to latent subgroup shifts with high-dimensional data.
Note that the identification arguments in the previous section imply that any joint distribution $p(\tilde U, C, X, Y, W)$ that satisfies our assumptions and matches the observed marginal distribution $p(C, X, Y, W)$ can be used to identify $q(Y \mid X)$.
We propose approximating such a joint distribution using a  model based on the Wasserstein Auto-Encoder \citep[WAE;][]{tolstikhin2018wasserstein}.
In this section, we describe modifications to the standard WAE to customize the learned joint distribution to our assumptions.

Formally, we approximate the true posterior $p(\tilde{U}|X,C,Y,W)$ with a recognition model or encoder $\hat{p}(\tilde{U}|X,C,Y,W)$ with parameters $\phi$. Given observed variables $\mathcal V = \{X, Y, C, W\}$, reconstruction loss  $\ell$, decoder $f$ with parameters $\theta$, divergence $D$, and prior distribution $\overline{p}(\tilde{U})$, the form of the training objective is
\begin{equation}\label{eq:vae_obj}
    \min_{\phi, \theta} \mathbb{E}_{p(\mathcal V)} \mathbb{E}_{{p} ( \tilde{U} \mid \mathcal V)}\big[\ell(\mathcal V, f(\tilde{U})) \big] + D(\hat{p}(\tilde{U}) \mid \mid \overline{p}(\tilde{U})).
\end{equation}

To encourage the inference network to learn a posterior distribution that conforms to Figure~\ref{fig:model}(c) we impose the following factorization on the joint probability
\begin{align*}
    p(\mathcal V,\tilde{U}) = p(Y|C,\tilde{U})p(C|X,\tilde{U})p(X|\tilde{U})p(W|\tilde{U})p(\tilde{U}).
\end{align*}
Given this, the  reconstruction (log) loss decomposes
\begin{align*}
    \ell(\mathcal V, f(\tilde{U})) &\;=  \beta_Y \ell_Y(Y, f_Y(C, \tilde{U})) + \beta_C \ell_C(C, f_C(X, \tilde{U})) \\ 
     &\quad+ \beta_X \ell_X(X, f_X(\tilde{U})) + \beta_W \ell_W(W, f_W(\tilde{U})).
\end{align*}
where the above subscripts indicate variable-specific decoders, loss functions, and scalar hyperparameter weights $\beta$.
As $\tilde{U}$ is discrete, to allow training with the reparameterization trick we model $\hat{p}(\tilde{U}|X,C,Y,W)$ using a Gumbel-Softmax distribution \citep{jang2016categorical,maddison2016concrete}.
We set the prior $\tilde{p}(\tilde{U})$ to be a uniform categorical distribution over the categories of $\tilde{U}$.

Given a trained WAE model, we can generate joint samples $\{(x_i,c_i,y_i,w_i,\tilde{u}_i)\}_{i=1}^n$ by the encoder $\hat p(\tilde U \mid X,C,Y,W)$.
Lemma~\ref{lem:u}, which establishes identification of this joint distribution under our assumptions, provides some justification for this approach.
All that remains to estimate are $p(\tilde{U}|X), q(\tilde{U})/p(\tilde{U}), p(Y|X, \tilde{U})$ and Equation~\eqref{eq:decomp}.
Each of these is readily estimable using standard classification models, as we have joint samples.
We discuss our implementation of this estimation strategy in the Appendix.
\section{Simulation Study} \label{simulation_study}
We now describe demonstrate our identification results in a simulated numerical examples.
These examples serve as a proof of concept that our identification strategies can serve as the basis for estimation methods.
In particular, we aim to show that (a) plug-in estimators based on our constructive proofs can be used to estimate $q(Y \mid X)$ in simple contexts, and (b) modifying deep latent variable models to respect the conditional independence structure in our setting can be an effective strategy for estimation in more complex settings.
We also show that estimators based on our adjustment strategy can succeed where standard covariate shift and label shift adaptation techniques, or naive applications of latent variable models, fail.

The simulations are structured as follows.
We have one source distribution $P$, and several target distributions $Q$, generated by latent subgroup shifts.
We train several models on the source distribution, some of which use unlabeled examples from Q for adaptation, then measure their performance on the target distribution.
In each case, we compare performance to two endpoints: the performance of an unadapted model trained by ERM on the source (ERM-SOURCE), which should be a lower bound on performance, and an oracle model trained directly on data from the target distribution (ERM-TARGET), which should be an upper bound.
We also compare to an oracle model that adjusts for $U$ using \eqref{eq:goal}, as if it were observed (LSA-ORACLE).

For these simulations, we fix a set of parameters that instantiate a case where standard empirical risk minimization (ERM-SOURCE) fails in a predictable way, while oracle adjustments for $U$ (LSA-ORACLE) recover the optimal target predictor $q(Y \mid X)$.
We do so by constructing a setting where the subgroup specific conditional expectation $E[Y \mid X, U]$ is sufficiently different across subgroups, thus producing a different ordering of predictions over examples from the target $q(Y \mid X)$.
Furthermore, we ensure that neither $U$ nor $Y$ can be perfectly reconstructed from $X$. If either were the case $p(Y \mid X) = p(Y \mid X, U) = q(Y \mid X, U) = q(Y \mid X)$, and the optimal predictor under $Q$ would simply correspond to the optimal predictor under $P$.
We then evaluate several estimation approaches based on our identification strategy (from which $U$ is hidden).

We sample datasets of size 10,000, and divide training, validation, and test sets into 70\%, 20\%, and 10\% splits.
For all experiments, we consider a fixed setting for the source distribution such that $p(U\!=\!1)\!=\!0.1$.
The target distribution varies over a range of settings of $q(U\!=\!1) \!\in\! \{0.1, 0.2, \dots, 0.9\}$.
Further details regarding the experimental procedure are provided in Appendix \ref{sec:appendix:experiments}.

\begin{figure*}[t!]
\centering
\includegraphics[width=0.9\textwidth]{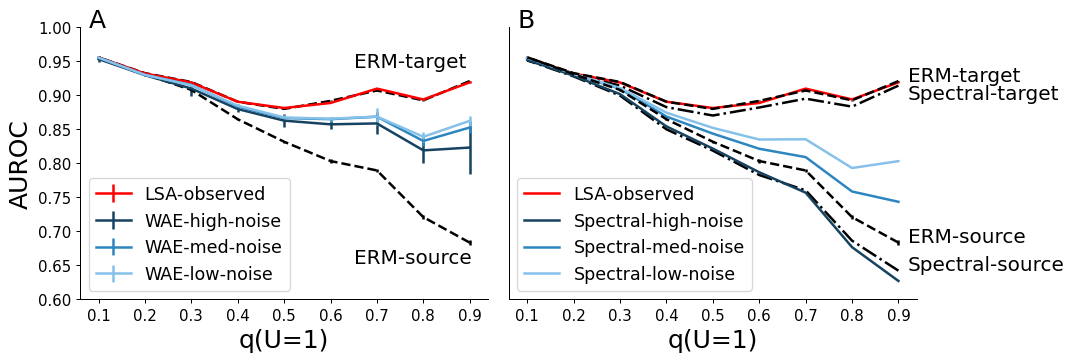}
\vspace{-2ex}
\caption{Simulation study: adaptation over target environments and varying levels of noise in the proxy variable $W$ ($\alpha_w=[1, 2, 3]$ for high, medium, and low noise). Shown is the mean $\pm$ std AUROC evaluated in varying target domains over ten training replicates for models trained in a fixed source domain ($p(U\!=\!1)\!=\!0.1$). {\bf Panel A} compares adaptation with Wassertein Autoencoders with structured decoders (LSA-WAE-S) and {\bf Panel B} compares adaptation with the continuous eigendecomposition (spectral) approach.}
\label{fig:synthetic_sweep}
\end{figure*}

To evaluate the discrete eigendecomposition approach (Algorithm \ref{algo:method}), we first apply K-means with two clusters to discretize X. 
The results in Table \ref{table:generation-properties} verify that the algorithm is capable of improving on estimates derived from the source domain in a setting where the magnitude of the distribution shift is large ($p(U\!=\!1)\!=\!0.1$ vs. $q(U\!=\!1)\!=\!0.9$) and $W$ is a noisy proxy of $U$ ($\alpha_w\!=\!1$).

\begin{table}[t!]
  \begin{center}
    \caption{Results of discrete simulation study ($\alpha_w \!=\! 1$, $n \!=\! 10^4$, $p(U\!=\!1)\!=\!0.1$, $q(U=1)\!=\!0.9$). Results shown are the RMSE between estimated and true $q(Y \mid X)$ across categories of discretized $X$.}
    \vspace{-1ex}
\begin{tabular}{lrr}
    \toprule  
        & RMSE  \\
  \midrule
 $p(Y|X)$ & $0.194$ \\
 ours & $\mathbf{0.056}$ \\
 \midrule
 $q(Y | X)$ & $0.004$ \\
   \bottomrule
\end{tabular}
\label{table:generation-properties}
\end{center}
\end{table}

For the case where $X$ is continuous, we compare the proposed adaptation approach to alternatives.
In the main text, we primarily evaluate performance using the area under the ROC curve (AUROC), but include analagous results in the appendix for the cross-entropy loss and accuracy (Supplementary Tables \ref{tab:supp:continuous:cross-entropy} and \ref{tab:supp:continuous:accuracy}).
In a setting analogous to the experiment conducted in the discrete case (Table \ref{table:continuous_synthetic}; $p(U\!=\!1)\!=\!0.1$, $q(U\!=\!1)\!=\!0.9$, $\alpha_w\!=\!1$)), models learned with ERM on the source domain (ERM-SOURCE) using a multilayer perceptron perform poorly in the target domain relative to those learned in the target domain (ERM-TARGET).
Furthermore, standard approaches to accounting for distribution shift, including covariate shift weighting (COVAR; \cite{shimodaira2000improving}), label shift weighting (LABEL; weighting by oracle $q(Y)/p(Y)$), and black box shift estimation (BBSE; \cite{lipton2018detecting}) do not outperform ERM-SOURCE.
However, we note that the latent shift adaptation approach with oracle access to $U$ (LSA-ORACLE; \eqref{eq:goal}) is able to perform on-par with ERM-TARGET without access to labeled data in the target domain.
Our main WAE-based approach that leverages the structured decoder and reconstruction loss (LSA-WAE-S) does not match LSA-ORACLE, but does partially mitigate the gap in performance between ERM-SOURCE and ERM-TARGET.
We compare to an alternative WAE specification that does not leverage a structured decoder (LSA-WAE-V) and find that it is does not improve on ERM-SOURCE.
This highlights the key role played by that the structural properties of the auxiliary variables $C$ and $W$.

We further evaluate the proposed WAE approach over varying degrees of distribution shift and levels of noise in the proxy variable $W$, and compare it to the continuous eigendecomposition method (appendix \ref{sec:appendix:continuous}) suggested by the proof of Theorem~\ref{thm:cont_q_y_x}.
We observe that ERM-SOURCE performance degrades smoothly as a function of the degree of distribution shift (Figure \ref{fig:synthetic_sweep}).
Both the WAE-based adaptation approach and the continuous eigendecomposition approach are capable of mitigating the performance degradation when the level of noise in $W$ is low ($\alpha_w\in \{2, 3\}$).
Overall, the WAE approach outperforms the continuous eigendecomposition approach and is less sensitive noise in $W$. In the high-noise setting ($\alpha_w=1$), the eigendecomposition approach is worse than the ERM-source but has similar performance to the eigendecomposition method without adaptation (the dashed line spectral-source in Figure~\ref{fig:synthetic_sweep}(b)).  

\begin{table}[t!]
    \centering
    \caption{Results of continuous simulation study ($\alpha_w\!=\!1$, $n \!=\! 10^4$, $p(U\!=\!1)\!=\!0.1$, $q(U\!=\!1)\!=\!0.9$), mean $\pm$ std AUROC over 10 random training replicates.}
    \begin{tabular}{lrr}
    \toprule
    Method         &    Source              &           Target       \\
    \midrule
    ERM-SOURCE     &  0.9560 $\pm$ 0.0001 &  0.6856 $\pm$ 0.0010 \\
    COVAR          &  0.9113 $\pm$ 0.0216 &  0.3274 $\pm$ 0.1351 \\
    LABEL          &  0.9561 $\pm$ 0.0001 &  0.6848 $\pm$ 0.0014 \\
    BBSE           &  0.9550 $\pm$ 0.0001 &  0.6789 $\pm$ 0.0005 \\
    LSA-WAE-S  &  0.9429 $\pm$ 0.0083 &  \textbf{0.8131 $\pm$ 0.0365} \\
    LSA-WAE-V      &  0.9550 $\pm$ 0.0006 &  0.6730 $\pm$ 0.0138 \\
    \midrule
    LSA-ORACLE     &  0.7843 $\pm$ 0.0254 &  0.9167 $\pm$ 0.0012 \\
    ERM-TARGET     &  0.7611 $\pm$ 0.0011 &  0.9194 $\pm$ 0.0001 \\
    \bottomrule
    \end{tabular}
    \label{table:continuous_synthetic}
\end{table}

\section{Discussion}

We presented a strategy for unsupervised domain adaptation under latent subgroup shift, which generalizes the standard settings of covariate and label shift.
Our strategy leverages auxiliary data in the source domain (concepts $C$ and a proxy $W$), and generalizes identification results from the causal inference literature to derive an identification strategy for the optimal predictor $q(Y | X)$ under the target distribution.
Our identification results are amenable to deep latent variable modeling, and suggest constraints that can be imposed on these models to make them effective for domain adaptation under this particular shift.
We demonstrated these claims in a carefully designed numerical example.

\paragraph{Limitations and future work} 
While a latent variable model has been shown promising to estimate the quantities of interest, such models are tricky to tune in practice, and have many known failure modes when used in causal contexts (see, e.g., \citet{rissanen2021critical}, who critique the method proposed in \citet{louizos2017causal}).
The identification arguments and corresponding modifications we make to the latent variable model may address some of these concerns, but practical challenges still remain.
For example, in practice, we observed that the dimensionality of the latent space mattered (the higher, the better) and that multiple preprocessing and training choices influenced the fit of the model (see Appendix for details).

Our approach requires the availability of mediating concepts $C$ and of a proxy variable $W$ at training time. This information might not be readily available, or it may not satisfy all the assumptions (e.g. $C$ such that $p(Y|C,U,X) = p(Y|C,U)$). Furthermore, these assumptions are typically not testable as $U$ is not observed.
However, we hope that our identification results can serve as motivation for careful collection of richer data, in which concepts and proxies may be present by design.

It is also worth deriving estimation guarantees (i.e., consistency guarantees, error bounds) for estimators of $q(Y|X)$. This would help understand if further data in $Q$ could improve estimation. For example, if we also observed $C$ in $Q$ would this more tightly bound the error of $q(Y|X)$?

We study the case where $U$ is discrete and other variables $\{W,Y,X,C\}$ can be either discrete or continuous. It is interesting to study the identification in the case that $U$ is continuous. In addition, our identification results require additional assumptions on the distribution, i.e., A\ref{asm:inversion}--A\ref{asm:mat_invert}, potentially limiting the class of distributions. These assumptions arise from the eigendecomposition technique used to show the  identification. It would interesting to understand whether these assumptions can be relaxed, perhaps incorporating results from proximal causal inference and missing data methods that do not need to identify the full joint distribution of observables and latent variables \citep[see, e.g.,][]{tchetgen2020introduction,kallus2021causal,li2021nonparametric}.


\paragraph{Acknowledgements}
We would like to thank Victor Veitch and Alexander Brown for valuable discussions and feedback. This work was funded by Google and supported by the Gatsby charitable foundation.

\bibliographystyle{unsrtnat}
\bibliography{references}


\newpage
\appendix
\onecolumn
\section*{Appendix}

\section{Proofs}

\subsection{Proof of Lemma~\ref{lem:u}}
Recall Lemma 1:

\textbf{Lemma 1.} \emph{Given that the above assumptions hold, all probability mass functions over discrete $\{W,X,C,Y,\tilde{U}\}$ in the source $P$ are identifiable, where $\tilde{U}$ is an unknown sorting of $U$.}

Before we prove this we will prove a variant of Theorem 1 of \citet{kuroki2014measurement}.

\begin{lemma}[variant of Theorem 1 of \citet{kuroki2014measurement}] \label{lemma:u_id}
Given A\ref{asm:graph}-A\ref{asm:inversion},  $p(W|\tilde{U})$ is identifiable.
\end{lemma}

\begin{proof}
First, fix a $k_U$. Without any additional information the easiest is to set $k_U \!=\! k_W$. However, if you believe that $k_U < k_W$, coarsen $W$ by dropping categories to ensure that the new dimensionality $k'_W$ is equal to $k_U$. Next notice that, given A\ref{asm:graph} (Figure~\ref{fig:model} (c)) we can factorize the joint of $W,X,Y$ conditional on $C$ as:
\begin{align*}
    p(Y, X, W \mid C) = \sum_{k=1}^{k_U} p(Y \mid C, U = k) p(X \mid C, U = k) p(W \mid  U = k) p(U = k \mid C).
\end{align*}
Next, construct the following matrices based on the decomposition of $p(Y, X, W | C)$ and of its marginal distributions:
\begin{align*}
    \bA &\;:=
    \begin{bmatrix}
        1 & p(W = 1 | C) & \cdots & p(W = k_W-1 |C) \\
        p(X = 1 | C) & p(X=1, W=1 | C) &  \cdots & p(X=1, W = k_W-1 | C) \\
        \vdots & \vdots & \ddots & \vdots \\
        p(X = k_X-1 | C) & p(X = k_X-1, W = 1 | C), & \cdots & p(X = k_X-1, W = k_W-1 | C)
    \end{bmatrix} \\
    \bB &\;:=
    \begin{bmatrix}
        p(Y | C) & p(Y, W=1 | C) & \cdots & p(Y, W = k_W-1 | C) \\
        p(Y, X=1 | C) & p(Y, X=1, W=1 | C) & \cdots & p(Y, X=1, W = k_W-1 | C) \\
        \vdots & \vdots & \ddots & \vdots \\
        p(Y, X = k_X-1 | C) & p(Y, X = k_X-1, W=1 | C) & \cdots & p(Y, X = k_X-1, W = k_W-1 | C)
    \end{bmatrix} \\
    \bR &\;:=
    \begin{bmatrix}
        1 & p(X=1 | C, U=1) & \cdots & p(X = k_X-1 | C, U=1) \\
        \vdots & \vdots & \ddots & \vdots \\
        1 & p(X=1 | C, U=k_U) & \cdots & p(X = k_X-1 | C, U=k_U)
    \end{bmatrix} \\
    \bM &\;:=
    \begin{bmatrix}
        p(U=1 | C) & 0 & \cdots & 0 \\
        & \ddots  &  &  \\
        0 &  \cdots & 0 & p(U = k_U | C)
    \end{bmatrix} \\
    \Lambda &\;:=
     \begin{bmatrix}
        p(Y | C, U=1 ) & 0 & \cdots & 0 \\
        & \ddots  &  &  \\
        0 &  \cdots & 0 & p(Y | C, U = k_U)
    \end{bmatrix} \\
    \bS &\;:=
    \begin{bmatrix}
        1 & p(W=1 | U=1) & \cdots & p(W = k_W-1 | U=1) \\
        \vdots & \vdots & \ddots & \vdots \\
        1 & p(W=1 | U=k_U) & \cdots & p(W = k_W-1 | U=k_U)
    \end{bmatrix}.
\end{align*}

Then note that
\begin{align}
    \bA = \bR^\top \bM \bS \;\;\;\;\;\;\;\;\;\;
    \bB = \bR^\top \bM \Lambda \bS. \label{eq:A_B_decomposition}
\end{align}

We then have that,

\begin{align}
\bA^{\dagger}\bB & =\left[\left(\bA^{\top}\bA\right)^{-1}\bA^{\top}\right]\bR^{\top}\bM\Lambda \bS \label{eq:decomp} \\
 & =\left(\bS^{\top}\bM\bR\bR^{\top}\bM\bS\right)^{-1}\bS^{\top}\bM\bR\left(\bR^{\top}\bM\Lambda \bS\right)\nonumber \\
 & \underset{(a)}{=}(\bS)^{-1}\Lambda \bS, \nonumber
\end{align}
where $\bA^{\dagger}$ is the Moore-Penrose pseudoinverse of $\bA$ (recall all pseudoinverses are unique and exist). Recall that above we have ensured that the dimensionality of $W$ is equal to the dimensional $U$. Thus, $\bS$ is square. Further, both $\bS$ and $\bR$ have rank at least $k_U$ by A\ref{asm:inversion}. So $\bS$ and $\bR\bR^{\top}$ are invertible. Because we have to marginalize $U$ in order to obtain observed distributions, it is only possible to identify $U$ up to an arbitrary permutation. Specifically, let  $\tilde{U}$ be a sorting of $U$ such that $p(Y|C,\tilde{U}=1) > p(Y|C,\tilde{U}=2) > \cdots > p(Y|C,\tilde{U}=k_U)$.

Now we need to show that we can obtain $p(W|\tilde{U})$ from eigendecomposition of $\bA^{\dagger}\bB$. To do so we first must solve $|\bA^{\dagger}\bB - \lambda \mathbf{I}| = 0$ for $\lambda$, to obtain the eigenvalues of $\bA^{\dagger}\bB$. Note that $|\bA^{\dagger}\bB - \lambda \mathbf{I}| = |(\bS)^{-1}\Lambda \bS - \lambda \mathbf{I}| = |\Lambda - \lambda \mathbf{I}| = 0$ where the second-to-last equality uses the Weinstein–Aronszajn identity $|(\bS)^{-1}\Lambda \bS - \lambda\mathbf{I}| =  |\bS(\bS)^{-1}\Lambda - \lambda\mathbf{I}| = |\Lambda - \lambda \mathbf{I}|$. Therefore, if we define $\lambda_1 > \cdots \lambda_{k_U}$ as the eigenvalues of $\bA^{\dagger}\bB$, it must be that $\lambda_i = p(Y | C, \tilde{U} = i)$ for $i = 1, \ldots, k_U$. 

Now that we have identified $p(Y | C, \tilde{U})$ we will show we can obtain $p(W | \tilde{U})$ from $\lambda_i$ and the eigenvectors $\eta_i$ of $\bA^{\dagger}\bB$. Define the matrix of eigenvectors as $\bH = [\eta_1, \ldots, \eta_{k_U}]$. To obtain this we must solve the linear system $\bA^{\dagger}\bB\bH = \bH \Lambda$. Note that $\bH$ is determined up to a multiplicative constant as $\lambda_1 \neq \cdots \neq \lambda_{k_U}$ from A\ref{asm:inversion}. Define a matrix of non-zero multiplicative constants $\bE = \mbox{diag}(\alpha_1, \ldots, \alpha_{k_U})$ and the shifted matrix $\bF = \bS^{-1}\bE$. Note that $\bA^{\dagger}\bB\bF = \bS^{-1}\Lambda \bS \bS^{-1} \bE = \bS^{-1}\Lambda\bE = \bS^{-1}\bE\Lambda = \bF \Lambda$. Therefore, $\bF$ is also a matrix of eigenvectors of $\bA^{\dagger}\bB$, and that $\bF = \bS^{-1}\bE = \bH$ for certain values of $\alpha_1, \ldots, \alpha_{k_U}$. To recover these, note that,
\begin{align*}
    \bS =
    \begin{bmatrix}
        1 & p(W=1 | U=1) & \cdots & p(W = k_W-1 | U=1) \\
        \vdots & \vdots & \ddots & \vdots \\
        1 & p(W=1 | U=k_U) & \cdots & p(W = k_W-1 | U=k_U)
    \end{bmatrix}
    = \bE\bH^{-1} =
    \begin{bmatrix}
        \alpha_1 h_{11} & \cdots & \alpha_1 h_{1 k_U} \\
        \vdots  & \ddots & \vdots \\
        \alpha_{k_U} h_{k_U 1} & \cdots & \alpha_{k_U} h_{k_U k_U}
    \end{bmatrix}.
\end{align*}
Equating the first column of both sides of the equation we have that $\alpha_1 = 1/h_{11}, \ldots, \alpha_{k_U} = 1/h_{k_U 1}$. This means that $\bS$ is identifiable from $\bE\bH^{-1}$ as $\bH^{-1}$ is what we estimate from eigendecomposition of $\bA^{\dagger}\bB$. Therefore, every element of $p(W | \tilde{U})$ is identifiable.
\end{proof}

Now that we have obtained $p(W|\tilde{U})$, we can prove Lemma~\ref{lem:u}.

\begin{proof}
As distributions that only involve $\{W,X,C,Y\}$ are observable, all we need to prove is that we can identify all distributions involving $\tilde{U}$. Let $\mathcal{V} \subseteq \{W,X,C,Y\}$ and $\mathcal{V}' \subseteq \{W,X,C,Y\} \setminus \mathcal{V}$. All we need to identify are 

\begin{enumerate}[label=(\alph*)]
\item $p(\tilde{U})$;
\item $p(\mathcal{V} \mid \tilde{U})$;
\item  $p(\tilde{U} \mid \mathcal{V})$;
\item $p(\mathcal{V} \mid \tilde{U}, \mathcal{V}')$.
\end{enumerate}

 Note that proving above identities are sufficient because (e) $p(\tilde{U}, \mathcal{V} \mid \mathcal{V}') = p(\mathcal{V} \mid \tilde{U}, \mathcal{V}')p(\tilde{U} \mid \mathcal{V}')$ (given by (d) and (c)). 

\paragraph{Identifying (a) $p(\tilde{U})$.} The identification is straightforward: note that $p(\bUt) = p(\bW|\bUt)^{\dagger}p(\bW)$.

\paragraph{Identifying (b) $p(\mathcal{V} \mid \tilde{U})$.} Recall we have already identified $p(\bW | \bUt)$. zlet $\mathcal{V}_{\setminus W} = \mathcal{V} \setminus W$. Note that $p(\boldsymbol{\mathcal{V}_{\setminus W}}, \bW | \bUt) =p(\boldsymbol{\mathcal{V}_{\setminus W}} | \bUt) p( \bW | \bUt)$ because $\mathcal{V}_{\setminus W} \ind W \mid \tilde{U}$. Hence, we have $$p(\boldsymbol{\mathcal{V}_{\setminus W}} \mid \bW) = p(\boldsymbol{\mathcal{V}_{\setminus W}} \mid \bUt)p(\bUt \mid \bW).$$
By multiplying $p(\bUt \mid \bW)^\dagger$ on both side, we can obtain
\begin{align*} 
p(\boldsymbol{\mathcal{V}_{\setminus W}} \mid \bUt) = p(\boldsymbol{\mathcal{V}_{\setminus W}} \mid \bW)p(\bUt \mid \bW)^{\dagger}
\end{align*}  Note that this is identified because the first term on the right-hand side is observed and the second term can be identified via Bayes rule $p(\tilde{U} | W) = p(W|\tilde{U})p(\tilde{U})/p(W)$, where $p(W|\tilde{U})$ is identifiable as shown in Lemma~\ref{lemma:u_id}.

\paragraph{Identifying (c) $p(\tilde{U} \mid \mathcal{V})$.} We have identified $p(\bUt | \bW)$ in the previous step using Bayes rule. We then have that $p(\bW | \boldsymbol{\mathcal{V}_{\setminus W}}) = p(\bW | \bUt)p(\bUt | \boldsymbol{\mathcal{V}_{\setminus W}}) \Rightarrow p(\bUt | \boldsymbol{\mathcal{V}_{\setminus W}}) = p(\bW | \bUt)^{\dagger}p(\bW | \boldsymbol{\mathcal{V}_{\setminus W}})$, which is identifiable. 
Finally we have via Bayes rule $p(\tilde{U} | \mathcal{V}_{\setminus W}, W) = p( \mathcal{V}_{\setminus W}, W | \tilde{U})p(\tilde{U}) / p(\mathcal{V}_{\setminus W}, W)$ all of which we can identify (via (a) and (b)).

\paragraph{Identifying (d) $p(\mathcal{V} \mid \tilde{U}, \mathcal{V}')$.} Note that $p(\boldsymbol{\mathcal{V}_{\setminus W}} | \mathcal{V}_{\setminus W}', \bW) = p(\boldsymbol{\mathcal{V}_{\setminus W}} | \bUt, \mathcal{V}_{\setminus W}')p(\bUt|\mathcal{V}_{\setminus W}', \bW) $, which implies that 
$$
p(\boldsymbol{\mathcal{V}_{\setminus W}} | \bUt, \mathcal{V}_{\setminus W}') = p(\boldsymbol{\mathcal{V}_{\setminus W}} | \mathcal{V}_{\setminus W}', \bW) p(\bUt|\mathcal{V}_{\setminus W}', \bW)^{\dagger}.
$$ 
The first term on the right-hand side is observed and the second is identified via (c). 
Finally note that $p({\mathcal{V}_{\setminus W}}, W | \tilde{{U}}, \mathcal{V}_{\setminus W}') = p(W | {\mathcal{V}_{\setminus W}}, \tilde{{U}}, \mathcal{V}_{\setminus W}')p({\mathcal{V}_{\setminus W}} | \tilde{{U}}, \mathcal{V}_{\setminus W}') = p(W | \tilde{{U}})p({\mathcal{V}_{\setminus W}} | \tilde{{U}}, \mathcal{V}_{\setminus W}')$ (as $W \ind \mathcal{V}_{\setminus W} | \tilde{U}$), where all right-hand terms are identified. Also that $p({\mathcal{V}_{\setminus W}}, | W, \tilde{{U}}, \mathcal{V}_{\setminus W}') = p({\mathcal{V}_{\setminus W}}, | \tilde{{U}}, \mathcal{V}_{\setminus W}')$ which is identified.
\end{proof}

\subsection{Proof of Theorem~\ref{thm:disc_q_y_x}}

\textbf{Theorem 1} (Discrete Observations). \emph{The distribution $q(Y|X)$ is identifiable from discrete $\{W,X,C,Y,\tilde{U}\} \sim P$ and $X \sim Q$.}

\begin{proof}
The first observation is that we can replace $U$ with $\tilde{U}$ everywhere. This is because $\tilde{U}$ has the exact same conditional independences as $U$ that are required in the factorization of $q(Y|X)$ in eq.~(\ref{eq:goal}) (as there are no requirements on the ordering of the categories of $U$). Further, we can replace $U$ with $\tilde{U}$ in eq.~(\ref{eq:goal}) without changing anything, i.e.,
\begin{align}\label{eq:q_x:factor}
    q(Y|X) \propto \sum_{i=1}^{k_U} p(Y|X,\tilde{U}=i)p(\tilde{U}=i|X)\frac{q(\tilde{U}=i)}{p(\tilde{U}=i)}.
\end{align}
This is because we are summing over categories of $U$ and so it makes no difference to change the order of categories of $U$, as in $\tilde{U}$. The only remaining thing to show is that $\frac{q(\tilde{U}=i)}{p(\tilde{U}=i)}$ can be identified. Note that
\begin{align*}
    \frac{q(X)}{p(X)} = \sum_{k=1}^{k_U} p(\tilde{U}=i | X) \frac{q(\tilde{U}=i)}{p(\tilde{U}=i)}.
\end{align*}
Define the vector $\bv_X = [q(X=1)/p(X=1), \ldots, q(X=k_X)/p(X=k_X)]$, the matrix $\bN_{ij} = p(\tilde{U}=i | X=j)$, and the vector $\bv_U = [q(U=1)/p(U=1), \ldots, q(U=k_X)/p(U=k_U)]$. We have that $\bv_U = \bN^{\dagger}\bv_X$. Note that $\frac{q(\tilde{U}=i)}{p(\tilde{U}=i)}$ is identified because $\bN, \bv_X$ are identified, and $\bN^{\dagger} = (\bN^\top \bN)^{-1}\bN^\top$  because $k_X \geq k_U$ by A\ref{asm:k_U} and (b) all linear systems have rank at least $k_U$ by A\ref{asm:inversion}.
\end{proof}

\subsection{Proof of Theorem~\ref{thm:cont_q_y_x}}
We first restate Theorem~\ref{thm:cont_q_y_x}:

\textbf{Theorem}~\ref{thm:cont_q_y_x} (Continuous Observations). \emph{The distribution $q(Y|X)$ is identifiable from continuous $\{W,X,C,Y\} \sim P$ and $X \sim Q$, and discrete $\tilde{U} \sim P$.}

\begin{proof}
The proof steps is similar to the proof of Theorem~\ref{thm:disc_q_y_x}: we can factorize the probability as~\eqref{eq:q_x:factor}. We identify each component as follows.

\paragraph{Identifying $p(W\mid \tilde{U})$.} We first show the continuous version of Lemma~\ref{lemma:u_id}.
As in the discrete case, given A\ref{asm:graph} we can factorize $p(Y,X,W|C)$ as written above. We rewrite it here in order to define functions $\bpsi_i(X),\bphi_i(W)$ and quantities $s_{i},m_{i}$ as follows,
\begin{align*}
    p(Y, X, W \mid C) = \sum_{k=1}^{k_U} \overbrace{p(Y \mid C, U = k)}^{m_{i}} \overbrace{p(X \mid C, U = i)}^{\bpsi_i(X)} \overbrace{p(W \mid  U = i)}^{\bphi_i(W)} \overbrace{p(U = i \mid C)}^{s_{i}}.
\end{align*}

To construct the integral operators for $\bbA,\bbB$ let $\mathcal{W},\mathcal{X}$ be the domains of $X,W$, respectively. Let $L_2(\mathcal{W},\mu)$ be the space of $L_2$-integrable functions on $\mathcal{W}$ with Lebesgue measure $\mu$ (and similarly for $\mathcal{X}$). Let $\bbA: L_2(\mathcal{W},\mu) \rightarrow L_2(\mathcal{X},\mu)$ and $\bbB: L_2(\mathcal{W},\mu) \rightarrow L_2(\mathcal{X},\mu)$ be the integral operators associated with kernel functions $p(X,W|C)$ and $p(Y,X,W|C)$, respectively. They are defined as 
\begin{align*}
    \bbA := \sum_{i=1}^{k_U} s_i \psi_i(X) \otimes \phi_i(W) \;\;\;\;\;\;\;\; \bbB := \sum_{i=1}^{k_U} s_i m_i \psi_i(X) \otimes \phi_i(W)
\end{align*} Note that these operators operate on any function $h \in L_2(\mathcal{W},\mu)$ in the following way, e.g., for $\bbA$,
\begin{align*}
    \bbA h = \sum_{i=1}^{k_U} s_i \psi_i(X) \langle \phi_i(W), h \rangle, \;\;\;\;\; s.t., \;\;\;\;\; \langle \phi_i(W), h \rangle := \int \phi_i(W)h(W) dW.
\end{align*}

Next we will describe how we can identify functions $p(Y|C,\tilde{U})$ and $p(W|\tilde{U})$ from eigendecomposition of the operator $\bbA^{\dagger} \bbB$. We begin by collecting functions into vectors/matrices that will make up this decomposition. Define the row vectors of functions 
\begin{align*}
\bpsi &:= [\psi_1(X), \ldots, \psi_{k_U}(X)];\\
\bphi &:= [\phi_1(W), \ldots, \phi_{k_U}(W)].
\end{align*}
 To fix the scale of the decomposition, we will apply the Gram–Schmidt process to $\bpsi,\bphi$ to create the set of orthonormal functions 
 \begin{align*}
    \overline{\bpsi} &:= [\overline{\psi_1}(X), \ldots, \overline{\psi_{k_U}}(X)] \\
    \overline{\bphi} &:= [\overline{\phi_1}(W), \ldots, \overline{\phi_{k_U}}(W)].
 \end{align*}
Note that the Gram-Schmidt process is well-defined as the inner product between functions is defined. This process also creates upper triangular matrices $\Rpsi,\Rphi\in\RR^{k_U\times k_U}$ that map the orthonormal functions back to their originals $\bpsi = \overline{\bpsi} \Rpsi$ and $\bphi = \overline{\bphi} \Rphi$. Finally define the diagonal matrices $\Lambda_s := \mbox{diag}(s_1, \ldots, s_{k_U})$ and $\Lambda_m := \mbox{diag}(m_1, \ldots, m_{k_U})$ Now note the following decompositions:
\begin{align}\label{eq:definition:AB}
    \bbA = \bpsi \Lambda_s \bphi^\top = \overline{\bpsi} \Rpsi \Lambda_s \Rphi^\top (\overline{\bphi})^\top \;\;\;\;\;\;\;\; \bbB = \bpsi \Lambda_s \Lambda_m \bphi^\top = \overline{\bpsi} \Rpsi \Lambda_s \Lambda_m \Rphi^\top (\overline{\bphi})^\top
\end{align}
Notice that the operator $\bbA$ is a finite-rank operator mapping between two finite dimensional spaces $\bbA: \mathbb{H}_\phi \rightarrow \mathbb{H}_\psi$, as $\mathbb{H}_\phi,\mathbb{H}_\psi$ are closed subspaces spanned by $\bpsi,\bphi$. Then we can write the inverse of $\bbA$ as:
\begin{align*}
    \bbA^{-1} = \overline{\bphi} (\Rphi^\top)^{-1} \Lambda_s^{-1} \Rpsi^{-1} (\overline{\bpsi})^\top.
\end{align*}
It follows that,
\begin{align*}
    \bbA^{-1} \bbB = \overline{\bphi} (\Rphi^\top)^{-1} \Lambda_s^{-1} \Rpsi^{-1} (\overline{\bpsi})^\top \overline{\bpsi} \Rpsi \Lambda_s \Lambda_m \Rphi^\top (\overline{\bphi})^\top = \overline{\bphi} (\Rphi^\top)^{-1} \Lambda_m \Rphi^\top (\overline{\bphi})^\top.
\end{align*}

Given the above decomposition, we now show that we can identify $\Lambda_m,\bphi$ and thus $p(Y|C,\tilde{U}),p(W|\tilde{U})$ via eigendecomposition. First notice that eigendecomposition of $\bbA^{-1} \bbB$ gives $\overline{\bphi} (\Rphi^\top)^{-1} \Lambda_m \Rphi^\top (\overline{\bphi})^\top$.
As in the discrete observation setting we have that the eigenvalues $\lambda_1, \ldots, \lambda_{k_U}$ must satisfy $|\bbA^{-1} \bbB - \lambda \mathbf{I} | = |\Lambda_m - \lambda \mathbf{I} | = 0$. Therefore, $\lambda_i = p(Y|C, \tilde{U}=i)$.
Using the same argument as we use in Theorem~\ref{lemma:u_id}, it follows that column of $\overline{\bphi} (\Rphi^\top)^{-1}$ are eigenfunctions of $\bbA^{-1}\bbB$.
Applying the Gram–Schmidt process to $\overline{\bphi} (\Rphi^\top)^{-1}$, we recover $\overline{\bphi}$ and $(\Rphi^\top)^{-1}$. We can then invert $(\Rphi^\top)^{-1}$ to identify $\bphi$ via $\bphi = \overline{\bphi} \Rphi$, and thus $p(W|\tilde{U})$.

All that is left to show is how to identify $p(\tilde{U}|X=x), p(Y|=x,\tilde{U}), q(\tilde{U})/p(\tilde{U})$. 

\paragraph{Identifying $p(\tilde{U}|X=x)$.} 
Under the A\ref{asm:graph} , we have $W\indep X\mid U$. Hence, we can write
\begin{equation}\label{eq:identification:fu_x}
p(W\mid X=x) = \sum_{k=1}^{k_U} p(W\mid \tilde{U}=i)p(\tilde{U}=k\mid X=x).
\end{equation}
By the linear independence condition stated in A\ref{asm:w_u_lindep}, we know that $f(W\mid X=x)$ is uniquely represented by  $f(W\mid \tilde{U}=1),\ldots,f(W\mid \tilde{U}=k_U)$. This implies for any $x\in\Domain({X})$, we can identify $p(\tilde{U}=i\mid X=x)$.

\paragraph{Identifying $p(Y|x,\tilde{U})$.}
Note that
\begin{align*}
p(W,Y\mid x) &= \sum_{k=1}^{k_U} p(Y\mid x,\tilde{U}=i)p(W\mid \tilde{U}=k)p(\tilde{U}=k\mid x)\\
&= {\bphi}\begin{bmatrix}p(\tilde{U}=1\mid x)&\cdots&0\\& \ddots&\\0&\cdots&p(\tilde{U}=k\mid x) \end{bmatrix}\begin{bmatrix}p(Y\mid x,\tilde{U}=1)\\\vdots\\ p(Y\mid x,\tilde{U}=k)\end{bmatrix}\\
& = \overline{\bphi}{R_\phi\begin{bmatrix}p(\tilde{U}=1\mid x)&\cdots&0\\& \ddots&\\0&\cdots&p(\tilde{U}=k\mid x) \end{bmatrix}\begin{bmatrix}p(Y\mid x,\tilde{U}=1)\\\vdots\\ p(Y\mid x,\tilde{U}=k)\end{bmatrix}}.
\end{align*}
Since $\overline{\bphi}_1,\ldots,\overline{\bphi}_k$ are pairwise orthonormal, it follows that for any $i \in \{1, \ldots, k\}$
\begin{equation}\label{eq:identification:py_ux}
\dotp{p(W,Y\mid X=x)}{\overline{\bphi}_i} = \int_{\Wcal}p(W,Y\mid X=x)\overline{\bphi}_i dw =  z_i(Y,x,\tilde{U}=i). 
\end{equation}

Then, we can obtain 
$$
\begin{bmatrix}p(Y\mid x,\tilde{U}=1)\\\vdots\\ p(Y\mid x,\tilde{U}=k)\end{bmatrix} = R_\phi^{-1}\begin{bmatrix}\frac{1}{p(\tilde{U}=1\mid x)}&\cdots&0\\& \ddots&\\0&\cdots&\frac{1}{p(\tilde{U}=k\mid x)} \end{bmatrix} \begin{bmatrix} z_i(Y,x,\tilde{U}=1)\\\vdots\\  z_k(Y,x,\tilde{U}=k)\end{bmatrix}
$$
as the inverse of $R_\phi$ exists given A\ref{asm:w_u_lindep}.

\paragraph{Identifying $q(\tilde{U})/p(\tilde{U})$.}
Note that
\begin{align*}
q(X)=\sum_{k=1}^{k_U} q(X\mid \tilde{U}=k) q(\tilde{U}=k)
&=\sum_{i=k}^{k_U} p(X\mid \tilde{U}=k) q(\tilde{U}=k)\\
&=\sum_{i=k}^{k_U} p(\tilde{U}=k, X)\frac{q(\tilde{U}=k)}{p(\tilde{U}=k)}\\
&=\sum_{i=k}^{k_U} p(\tilde{U}=k|X)p(X)\frac{q(\tilde{U}=k)}{p(\tilde{U}=k)}.
\end{align*}
This implies that for all $x\in\Domain(X)$
\[
\frac{q(x)}{p(x)} = \sum_{k=1}^{k_U} p(\tilde{U}=k|x)\frac{q(\tilde{U}=k)}{p(\tilde{U}=k)}.
\]
Now select observed $x_1,\ldots,x_k$ that satisfies A\ref{asm:mat_invert}. Then, we can write 
\begin{align*}
    \underbrace{\begin{bmatrix}
    \frac{q(x_1)}{p(x_1)}\\
    \frac{q(x_2)}{p(x_2)}\\
    \vdots\\
    \frac{q(x_k)}{p(x_k)}
    \end{bmatrix}}_{\bv_{q,p,X}}
    =
    \underbrace{\begin{bmatrix}
    p(\tilde{U}=1\mid x_1)& p(\tilde{U}=2\mid x_1)&\cdots
    &p(\tilde{U}=k\mid x_1)\\
    p(\tilde{U}=1\mid x_2)& p(\tilde{U}=2\mid x_2)&\cdots
    &p(\tilde{U}=k\mid x_2)\\
    \vdots&\ddots&&\vdots\\
    p(\tilde{U}=1\mid x_k)& p(\tilde{U}=2\mid x_k)&\cdots
    &p(\tilde{U}=k\mid x_k)\\   
    \end{bmatrix}}_{\bM_{\tilde{U},X}}
    \underbrace{\begin{bmatrix}
    \frac{q(\tilde{U}=1)}{p(\tilde{U}=1)}\\
    \frac{q(\tilde{U}=2)}{p(\tilde{U}=2)}\\
    \vdots\\
    \frac{q(\tilde{U}=k)}{p(\tilde{U}=k)}
    \end{bmatrix}}_{\bv_{q,p,\tilde{U}}}.
\end{align*}
By A\ref{asm:mat_invert}, the confusion matrix  $\bM_{\tilde{U},X}$ is invertible and hence we can obtain $q(\tilde{U})/p(\tilde{U})$ via $\bv_{q,p,\tilde{U}} = \bM_{\tilde{U},X}^{-1}\bv_{q,p,X}$, and we are done.
\end{proof}

\section{Experimental details} \label{sec:appendix:experiments}
Here we describe the construction of the simulation study considered in Section \ref{simulation_study}.
We let $k_U \!=\! 2, k_X \!=\! 2, k_C \!=\! 3, k_Y \!=\! 2, k_W \!=\! 2$, where $X$ is continuous and $U$, $Y$, $C$, and $W$ are discrete. 
We generate $C$ as a multilabel variable where each dimension $C_j$ takes on a value of either 0 or 1, giving a discrete variable with $2^{k_C}$ states.
Let $\bo(v)$ be the $|V|$-dimensional one-hot representation of a sample from a categorical variable $v \in V$. 
Let $V_j$ designate the $j$-th dimension of a categorical random variable $V$.
Let ${\bf I}_k$ be the identity matrix of size $k \times k$.
Let $\mathrm{sign}$ be the function such that $\mathrm{sign}(z)=1$ if $z >0$ and $\mathrm{sign}(z)=0$ otherwise.
For a vector $\boldsymbol{\pi}$ drawn from the $(k_{U}\!-\!1)$-dimensional simplex, the data are simulated as
\begin{align*}
    U  \sim & \; \textrm{Categorical}(\boldsymbol{\pi}) \\
    W \mid U=u \sim & \; \mathrm{sign} \big( \mathcal{N}(\bo(u) \mathbf{M}_{W|U}, 1 \big) \\
    X \mid U=u \sim &\; \mathcal{N}(\bo(u) \mathbf{M}_{X|U}, \mathbf{ I}_{k_X}) \\
    C_j \mid X=x,U=u \sim &\; \textrm{Bernoulli} \Big(\mathrm{logit}^{-1} \big( x \mathbf{M}_{C|X,U=u} + \bo(u) \mathbf{M}_{C|U} \big) \Big) \\
    Y \mid C=c, U=u \sim &\; \textrm{Bernoulli} \Big ( \mathrm{logit}^{-1} \big(c \mathbf{M}_{Y|C, U=u} + \bo(u) \mathbf{M}_{Y|U} \big) \Big),
\end{align*}

where the matrices are defined as
\begin{align*}
    \;& \mathbf{M}_{W|U} :=
    \alpha_w \begin{bmatrix}
        -1 & 1
    \end{bmatrix}^\top
    \hspace{2mm}
\mathbf{M}_{X|U} :=
     \begin{bmatrix}
        -1 & 1 \\
        1 & -1
    \end{bmatrix}
    \hspace{2mm}
    \mathbf{M}_{C|U} :=
    \begin{bmatrix}
        -2 & 2 & 2 \\
        -1 & 1 & 2
    \end{bmatrix} \\
    \;& \mathbf{M}_{C|X,U=u_0} :=
    3\begin{bmatrix}
        -2 & 2 & -1 \\
        1 & -2 & -3
    \end{bmatrix} 
    \hspace{2mm}
    \mathbf{M}_{C|X,U=u_1} :=
    3\begin{bmatrix}
        2 & -2 & 1 \\
        -1 & 2 & 3
    \end{bmatrix} \\
    \;& \mathbf{M}_{Y|U} :=
    \begin{bmatrix}
        2 & 2
    \end{bmatrix}^\top 
    \mathbf{M}_{Y|C, U=u_0}  := 
    \begin{bmatrix}
        3 & -2 & -1
    \end{bmatrix} ^\top
    \mathbf{M}_{Y|C, U=u_1} := 
    \begin{bmatrix}
        3 & -1 & -2
    \end{bmatrix} ^\top.
\end{align*}

\begin{table}[t!]
\parbox{.45\linewidth}{
\centering
\caption{Cross-entropy for continuous observations ($\alpha_w=1$, $n \!=\! 10^4$), mean $\pm$ std over 10 training replicates.}
\label{tab:supp:continuous:cross-entropy}
\begin{tabular}{lrr}
\toprule
Method         &    Source              &           Target       \\
\midrule
ERM-SOURCE     &  0.1683 $\pm$ 0.0002 &  0.3979 $\pm$ 0.0007 \\
COVAR          &  0.2489 $\pm$ 0.0107 &  0.4206 $\pm$ 0.0711 \\
LABEL          &  0.1726 $\pm$ 0.0021 &  0.3635 $\pm$ 0.0111 \\
BBSE           &  0.2372 $\pm$ 0.0065 &  0.8461 $\pm$ 0.0233 \\
LSA-WAE-S &  0.1962 $\pm$ 0.0112 &  0.2530 $\pm$ 0.0248 \\
LSA-WAE-V      &  0.1751 $\pm$ 0.0112 &  0.3929 $\pm$ 0.0409 \\
\midrule
LSA-ORACLE     &  0.3300 $\pm$ 0.0250 &  0.1637 $\pm$ 0.0008 \\
ERM-TARGET     &  0.3415 $\pm$ 0.0010 &  0.1660 $\pm$ 0.0003 \\
\bottomrule
\end{tabular}
}
\hfill
\parbox{.45\linewidth}{

\caption{Accuracy for continuous observations ($\alpha_w=1$, $n \!=\! 10^4$), mean $\pm$ std over 10 training replicates.}
\label{tab:supp:continuous:accuracy}
\centering 
\begin{tabular}{lrr}
\toprule
Method         &    Source              &           Target       \\
\midrule
ERM-SOURCE     &  0.9179 $\pm$ 0.0006 &  0.7972 $\pm$ 0.0009 \\
COVAR          &  0.8807 $\pm$ 0.0011 &  0.9199 $\pm$ 0.0140 \\
LABEL          &  0.9153 $\pm$ 0.0011 &  0.8294 $\pm$ 0.0115 \\
BBSE           &  0.8935 $\pm$ 0.0030 &  0.5875 $\pm$ 0.0083 \\
LSA-WAE-S      &  0.8994 $\pm$ 0.0093 &  0.8924 $\pm$ 0.0181 \\
LSA-WAE-V      &  0.9121 $\pm$ 0.0113 &  0.7942 $\pm$ 0.0489 \\
\midrule
LSA-ORACLE     &  0.8555 $\pm$ 0.0225 &  0.9320 $\pm$ 0.0008 \\
ERM-TARGET     &  0.8653 $\pm$ 0.0030 &  0.9342 $\pm$ 0.0005 \\
\bottomrule
\end{tabular}
}
\end{table} 

To construct the setting used for the simulation experiments, we draw a sample from a source domain where $\boldsymbol{\pi}$ is such that $p(U=1)=0.1$. We further draw several target distributions where $\boldsymbol{\pi}$ is such that $q(U=1) \in \{0.1, 0.2, \dots, 0.9\}$. 
We vary the noisiness of the proxy $W$ by generating three copies of the target domain datasets where $\alpha_w \in \{1, 2, 3\}$ such that greater values for $\alpha_w$ indicate less noise.

For the ERM baselines considered for the experiment presented in Tables \ref{table:continuous_synthetic}, \ref{tab:supp:continuous:cross-entropy}, and \ref{tab:supp:continuous:accuracy}, and Figure \ref{fig:synthetic_sweep} we use a multilayer perceptron (MLP) with one hidden layer of size 100 with ReLU activations. We train for 200 epochs with a batch size of 128 using stochastic gradient descent (SGD) with a learning rate of 0.01 that is reduced by a factor of ten if the training loss has not improved by at least 0.01 in the last 20 epochs, with a minimum learning rate of $10^{-7}$. We use a weight decay of $10^{-6}$.
The training procedure is implemented using Tensorflow 2.12.0.

For the covariate shift adjustment baseline, we fit a domain classifier, using the same model architecture and training procedure in the model for $Y$, derive instance weights following \cite{shimodaira2000improving}, and apply weighted ERM with the same procedure as in the unweighted case.
The label shift baseline with oracle access to labels in the target domain (LABEL) applies weighted ERM with learned class weights $q(Y)/p(Y)$ based on observed frequencies in the validation set in the source domain and the training set in the target domain.
For the BBSE approach \citep{lipton2018detecting}, where the labels $Y$ are not available in the target domain, we first fit an auxiliary model with ERM to estimate $p(Y \mid X)$ in the source domain, using the same procedure as before, and use its predictions on the source validation set and target training set to estimate $q(Y)/p(Y)$ using the soft confusion matrix approach of \citep{garg2020unified}.
We clip weights derived from the confusion matrix approach to the range $[0.01, 15]$.
For the adjustment procedure that implements equation (\ref{eq:goal}) with oracle access to $U$ (LSA-ORACLE), we fit auxiliary models for $p(Y \mid X, U)$ and $p(U \mid X)$, using the model for $p(U \mid X)$ directly in equation (\ref{eq:goal}) and as the predictor used to derive $q(U)/p(U)$ with the soft confusion matrix approach.
We apply temperature scaling as an additional calibration step to each auxiliary model and the final result of each procedure.
The temperature scaling procedure is implemented as a uniform scaling of the output logits by a scalar learned on the validation data using SGD with a fixed learning rate of 0.001. 

For the WAE-based adaptation approach, we use an encoder with one hidden layer of size 100 and set the dimensionality of the learned latent space over $\tilde{U}$ to be 10. Following the construction in section \ref{sect:estimation}, we use a model architecture and objective function that reflects the factorization of the joint distribution implied by the causal graph (LSA-WAE-S). 
For this approach, we use separate decoder networks $\{f_Y, f_C, f_X, f_W\}$ of one hidden layer of size 100 for each of the observed variables. 
We use categorical cross-entropy losses over the reconstruction of $Y$ and $W$ and the elementwise binary cross-entropy loss over the elements of $C$.
The loss $\ell_X$ over $X$ is given by $\log(\sigma_X) + \frac{1}{\sigma_X} (X - f_X(\tilde{U}))^2$, where $\sigma_X$ is a learned parameter.
The weight $\beta$ on reconstruction loss associated with each of $C$, $W$, and $Y$ is the reciprocal of the entropy of the variable, estimated on the training data of the source domain, and the weight $\beta_X$ is analogously the reciprocal of the variance of $X$.
The KL divergence term in the loss is weighted by a factor of 3.
The WAE is fit using the RMSprop optimizer for 200 epochs using a learning rate of $10^{-4}$, annealed with the same strategy as in the baseline approaches.
We anneal the temperature of the Gumbel-softmax distribution used for sampling $\tilde{U}$ by a factor of 0.9999 at each training iteration, starting from an initial temperature of 1 to a minimum temperature of 0.01.

\section{Estimation procedure for continuous random variables}\label{sec:appendix:continuous}

In this section, we introduce the estimation procedure for continuous random variables, an extension of Algorithm~\ref{algo:method}. The continuous setting requires an additional step to select $k_U$ points from the domain of $X$ such that the constructed confusion matrix is invertible. While there exist various density function estimators, the main challenge is finding a reliable density estimator for computing of the underlying eigenfunctions. To this end, we employ the Least-Squares Conditional Density Estimator (LS-CDE)~\citep{sugiyama2010conditional}, where the set of basis functions are pre-defined by users. This method allows us to easily compute the eigenfunctions of the underlying density operators, which, in turn are a finite set of basis functions. The complete estimation procedure is presented in Algorithm~\ref{algo:method:continuous:main}. The algorithm is implemented for discrete $U$, $Y$, $C$ and continuous $X$, $W$, which matches the simulation setting in Section~\ref{simulation_study}. 
We first briefly introduce the LS-CDE method and discuss the selection of basis functions, followed by  the details of each step.

\begin{algorithm}[t]
\caption{Estimating Continuous $q(Y|\xnew)$ .}\label{algo:method:continuous:main}
\begin{algorithmic}[1]
  \Require
  source $\{(w_i, x_i, c_i, y_i)\}_{i=1}^n$;
  target $\{\tilde{x}_j\}_{j=1}^m$; Given $\xnew$ 
  \State $\hat{p}(W|\tilde{U})\leftarrow$  Algorithm~\ref{algo:continuous:pw_u}$( \{(w_i, x_i, c_i, y_i)\}_{i=1}^n)$
  \State $[\hat{q}(\tilde{\bf U})/\hat{p}(\tilde{\bf U})]\leftarrow$ Algorithm~\ref{algo:continuous:qu_pu}$(\{(w_i, x_i)\}_{i=1}^n,\;\{\tilde{x}_j\}_{j=1}^m,\; \{\hat{p}(W|\tilde{U}=k)\}_{k=1}^{k_U} )$
   \State $ \hat{p}(\bUt\mid\xnew)$ is obtained by solving~\eqref{eq:lsq:pux} 
  \State $\hat{p}({\bf Y} \mid \tilde{\bf U}, \xnew)\leftarrow$ Algorithm~\ref{algo:continuous:pwy_x}$(\{(w_i, x_i, y_i)\}_{i=1}^n, \;\{\hat{p}(W|\tilde{U}=k)\}_{k=1}^{k_U} , \hat{p}(\bUt\mid\xnew))$
  \For{$y=1,\ldots,k_Y$}
   \State  $\hat{q}(y\mid\xnew)\leftarrow\sum_{i=1}^{k_U}\hat{p}(y\mid \xnew, \tilde{U}=i)\hat{p}(\tilde{U}=i\mid \xnew){\frac{{q}{(\tilde{U})}}{{p}(\tilde{U})}}$
  \EndFor
  \State  $\hat{q}({\bf Y}\mid \xnew)\leftarrow\hat{q}({\bf Y}\mid \xnew)/\sum_{i=1}^{k_y}\hat{q}(y\mid\xnew)$
\end{algorithmic}
\end{algorithm}

\subsection{Brief introduction of least-squares conditional density estimator}\label{appendix:intro:lscde}

Given a pair of random variables ($X$, $Y$), the Least-Squares Conditional Density Estimator (LS-CDE))~\citep{sugiyama2010conditional} assumes the following form 
\[
p(Y\mid X) = \frac{p(X,Y)}{p(X)}:= r(X,Y),
\]
where $r(X,Y)$ is the density ratio function. 
Let $\{g_1(x,y), \ldots, g_m(x,y)\}$ to be a set of basis functions such that (1) $g_i(x,y)\geq 0$ for every $i\in[m]$ and $x\in\Domain(X)$ and $y\in\Domain(Y)$. To estimate $r(X,Y)$, we consider the estimate $\hat{r}_{\bm\alpha}(X,Y)$ that lies in the linear subspace of $r(x,y)\in\{{\bm\alpha}^\top{\bf g}(x,y):{\bm\alpha}\in{\RR}^m\}$ with ${\bf g}(x,y) = (g_1(x,y),\ldots,g_m(x,y))$. Hence, the goal is to estimate the coefficient vector ${\bm\alpha}$ from data. To this end, \citet{sugiyama2010conditional} proposed the following objective functional:
\[
 \arg\min \frac{1}{2}\int\int\rbr{{\bm\alpha}^\top{\bf g}(x,y) - \frac{p(x,y)}{p(x)}}^2p(x)dxdy.
\]
With simple algebraic manipulation, the above objective function is equivalent as the following:
\[
\hat{\bm\alpha}=\arg\min \frac{1}{2}{\bm\alpha}^\top{\bf H}{\bm\alpha} - {\bf h}^\top{\bm\alpha}, 
\]
where
\[
{\bf H} := \int\int {\bf g}(x,y){\bf g}(x,y)^\top p(x) dydx,\quad {\bf h}:=\int\int {\bf g}(x,y)p(x,y) dx dy.
\]
Since the density functions $p(x,y)$ and $p(x)$ are unknown, we can compute the empirical estimators of ${\bf H}$ and ${\bf h}$ from independent samples $\{(x_i,y_i)\}_{i=1}^n$ as follows:
\[
\hat{\bf H} := \frac{1}{n}\sum_{i=1}^n \int {\bf g}(x_i,y){\bf g}(x_i, y)^\top dy,\quad \hat{\bf h} := \frac{1}{n}\sum_{i=1}^n {\bf g} (x_i,y_i).
\]
To stabilize the empirical estimator, we additionally add a regularizer $\lambda{\bm\alpha}^\top{\bm\alpha}$ with $\lambda>0.$ The overall objective function is summarized as

\begin{equation}\label{eq:obj}
\tilde{\bm\alpha}:=\arg\min \frac{1}{2}{\bm\alpha}^\top\hat{\bf H}{\bm\alpha} - \hat{\bf h}^\top{\bm\alpha} + \lambda{\bm\alpha}^\top{\bm\alpha}.
\end{equation}
Note that~\eqref{eq:obj} is a quadratic program and yields an analytical solution
\[
\tilde{\bm\alpha} = (\hat{\bf H} + \lambda{\bf I})^{-1}\hat{\bf h}.
\]

To ensure the the estimated conditional density is non-negative everywhere, we output $\hat{\bm \alpha}=(\hat\alpha_1,\ldots,\hat\alpha_m)$ such that $\hat\alpha_i=\max(0, \tilde{\alpha}_i)$ for $i\in[m]$.
In our simulations, we found that choosing $\lambda=10^{-2}$ suffices to provide good results.

Choosing the candidate basis functions requires knowledge of the underlying distributions. When the class of distribution is unknown, Gaussian kernel functions can often be used as a basis to provide a good approximation of the distribution(s). In addition, the Gaussian kernel function yields and analytical result for the integral $\hat{\bf H}$. Specifically, let $g_\ell(x,y)=\exp(-\|(x-x_\ell\|^2+\|y-y_\ell\|^2)/2\sigma^2)$ and $g_{\ell'}(x,y)=\exp(-(\|x-x_{\ell'}\|^2+\|y-y_{\ell'}\|^2)/2\sigma^2)$ for some $x_\ell,x_{\ell'}\in\Domain(X)$, $y_\ell,y_{\ell'}\in\Domain(Y)$ and $\sigma>0$, we have
\[
\int g_\ell(x,y)g_{\ell'}(x,y)dy = (\sqrt{\pi}\sigma)^{d_y}\exp\rbr{-\frac{\|y_\ell-y_{\ell'}\|^2}{4\sigma^2}}
\exp\rbr{-\frac{\|x-x_{\ell'}\|^2+\|x-x_{\ell}\|^2}{2\sigma^2}}
,
\]
where $d_y$ is the dimension of $Y$. Hence, we do not need to resort to numerical methods to compute $\hat{\bf H}$.

\subsection{Implementation details of Algorithm~\ref{algo:continuous:pw_u}}\label{appendix:step:continuous:pw_u}

\begin{algorithm}[t]
\caption{Estimate Continuous $p(W|U)$, details provided in Section~\ref{appendix:step:continuous:pw_u}}\label{algo:continuous:pw_u} 
\begin{algorithmic}[1]
  \Require
  source $\mathcal{P} \!=\! \{(w_i, x_i, c_i, y_i)\}_{i=1}^n$; Given $c\in\Domain(C)$ and $y\in\Domain(Y)$  
  \State { $\hat{p}(W,X\mid c)$ and $\hat{p}(W,X,y\mid c)$ via least-squares density estimator~\eqref{eq:lsde:pwx_c}--\eqref{eq:lsde:pwxy_c}}
  \State \text{Find the decomposition $\hat{p}(W,X\mid c)$ and $\hat{p}(W,X,y\mid c)$~\eqref{eq:decomposition}}
  \State Eigendecompose $\hat{\bbA}^{-1}\hat\bbB'$~\eqref{eq:A_invB} and obtain the eigenfunctions 
  \State \text{Compute the inverse of the eigenfunctions to obtain $\{\hat{p}(W|\tilde{U}=1),\ldots,\hat{p}(W|\tilde{U}=k_U)\}$
  }
\end{algorithmic}
\end{algorithm}

In this section, we introduce the implementation details of Algorithm~\ref{algo:continuous:pw_u} step-by-step. 

\emph{ Step 1 of Algorithm~\ref{algo:continuous:pw_u}}.
Since both $Y,C$ are discrete random variables, for any fixed $c\in\Domain(C),y\in\Domain(Y)$, the conditional density functions $\hat{p}(w,x\mid c)$ and $\hat{p}(w,x,y\mid c)$ can be estimated by marginal density estimators. We use the least-squares density estimator to estimate both $\hat{p}(w,x\mid c)$ and $\hat{p}(w,x,y\mid c)$ with Gaussian kernel basis functions of length-scale $1$: 
\begin{equation}\label{eq:basisfunctions}
\bigg\{g_\ell(w,x)=\varphi_\ell(w)\vartheta_\ell(x):\varphi_\ell(w)=\exp\rbr{-\frac{\|w-\bar{w}_\ell\|^2}{2}},\vartheta_\ell(x)=\exp\rbr{-\frac{\|x-\bar{x}_\ell\|^2}{2}},\ell=1,\ldots,k_U\bigg\},
\end{equation}
where the centers $\bar{x}_{\ell}$, $\bar{w}_\ell$ for $\ell=1,\ldots,m$ are chosen to match the means of the mixture models from the data generation process, namely ${\bf M}_{W\mid U}$ and ${\bf M}_{X\mid U}$ defined in Section~\ref{sec:appendix:experiments}. Here, we assume the density functions have the following form
\[
{p}(w,x|c) = {\bm \alpha}_c^\top{\bf g}(w,x), \quad {p}(w,x,y\mid c) = p(w,x\mid y,c){p}(y\mid c)={\bm\beta}_{y,c}^\top {\bf g}(w,x){p}(y\mid c),
\]
where the conditional probability $p(y\mid c)$ can be seen as a constant given that $y,c$ are fixed. 
Hence, it is natural to assume that the empirical marginal density estimator has the following form
\[
\hat{p}(w,x|c) = \hat{\bm \alpha}_c^\top{\bf g}(w,x), \quad \hat{p}(w,x\mid y,c) = \hat{\bm\beta}_{y,c}^\top {\bf g}(w,x),
\]
 We obtain  coefficient vectors $\hat{\bm \alpha}_c$ and $\hat{\bm\beta}_{y,c}$ by solving a similar objective function as LS-CDE:
\begin{align*}
{\bm\alpha}_c &=\argmin \frac{1}{2}\int\int \rbr{{\bm \alpha}_c^\top{\bf g}(w,x) - p(w,x\mid c)}^2dwdx;\\
{\bm\beta}_{y,c} &=\argmin \frac{1}{2}\int\int \rbr{{\bm \beta}_{y,c}^\top{\bf g}(w,x) - p(w,x\mid y,c)}^2dwdx.
\end{align*}
Given subsets of samples $\{(x_i,w_i,c_i)\}_{i\in \Ncal_c}$ with $\Ncal_c=\{i\in [n]:c_i=c\}$ and $\{(x_i,y_i,w_i,c_i)\}_{i\in \Ncal_{y,c}}$ with $\Ncal_{y,c}=\{i\in[n]: y_i=y,c_i=c\}$ from the original sample set $\{(x_i,y_i,w_i,c_i)\}_{i=1}^n$, we can construct the associated regularized empirical estimators. Define 
$\tilde{\bf H} = \int\int{\bf g}(x,y){\bf g}(x,y)^\top dxdy$, $\hat{\bm\alpha}_c$, $\hat{\bm\beta}_{y,c}$ are obtained by solving the following function 
\begin{align}
\hat{\bm\alpha}_c &= \arg\min \frac{1}{2}{\bm\alpha}^\top\tilde{\bf H}{\bm\alpha} - \hat{\bf h}_c^\top{\bm\alpha}+\lambda{\bm\alpha}^\top{\bm\alpha},\quad
\hat{\bf h}_c =\frac{1}{|\Ncal_c|}\sum_{i\in\Ncal_c} {\bf g}(w_i,x_i);\label{eq:lsde:pwx_c} \\
\hat{\bm\beta}_{y,c} &= \arg\min \frac{1}{2}{\bm\beta}^\top\tilde{\bf H}{\bm\beta} - \hat{\bf h}_{y,c}^\top{\bm \beta} + \lambda{\bm\beta}^\top{\bm\beta},\quad \hat{\bf h}_{y,c} = \frac{1}{|\Ncal_{y,c}|}\sum_{i\in\Ncal_{y,c}} {\bf g}(w_i,x_i).\label{eq:lsde:pwxy_c}
\end{align}
It is worth noting that the integral of Gaussian kernel functions $\tilde{\bf H} = \int\int{\bf g}(w,x){\bf g}(w,x)^\top dxdw$ has an analytical form and the objective function is quadratic, yielding analytical forms $\hat{\bm\alpha}_c=(\tilde{\bf H}+\lambda{\bf I})^{-1}\hat{\bf h}_c$ and $\hat{\bm\beta}_{y,c}=(\tilde{\bf H}+\lambda{\bf I})^{-1}\hat{\bf h}_{y,c}$.

\emph{ Step 2--3 of Algorithm~\ref{algo:continuous:pw_u}}. With the estimated from \emph{Step 1}, we can construct the empirical integral operator $\hat\bbA$ and $\hat\bbB'$ with respect to the kernel functions $\hat{p}(w,x\mid c)$ and $\hat{p}(w,x\mid y,c)$, respectively, as
\begin{equation}\label{eq:decomposition}
\hat\bbA = \sum_{i=1}^{k_U}\hat\alpha_{c,i}\vartheta_i(X)\otimes\varphi_i(W),\quad \hat\bbB' = \sum_{i=1}^{k_U}\hat\beta_{y,c,i}\vartheta_i(X)\otimes\varphi_i(W).
\end{equation}
To find the inverse of $\hat\bbA$, we first run the Gram-Schmidt procedure on $\{\vartheta_1,\ldots,\vartheta_{k_U}\}$ and $\{\varphi_1,\ldots,\varphi_{k_U}\}$ respectively to orthonormalize the basis functions. Since we are using Guassian kernels, the Gram-Schmidt procedure can be obtained analytically. We provide the example of constructing the first two orthonormal components and the rest of them can be constructed similarly. We have
\begin{align*}
    \overline{\vartheta}_1&=\frac{\vartheta_1}{\norm{\vartheta_1}},&&
    \dotp{\vartheta_1}{\vartheta_1} = \int \vartheta_1(x)\vartheta_1(x)dx = (\sqrt{\pi})^{d_x}; 
    \\
    \overline{\vartheta}_2&=\frac{\vartheta_2-\dotp{\vartheta_2}{\overline{\vartheta}_1}}{\norm{\vartheta_2-\dotp{\vartheta_2}{\overline{\vartheta}_1}}},&& \dotp{\vartheta_2}{\overline{\vartheta}_1} = \int \overline\vartheta_1(x)\vartheta_2(x)dx 
    = (\sqrt{\pi})^{d_x/2}\exp\rbr{-\frac{\norm{\bar{x}_1-\bar{x}_2}^2}{4}}.
\end{align*}
Let ${\bf R}_{\vartheta}\in\RR^{k_U\times k_U}$ be a coefficient matrix whose $ij$-th entry is $\dotp{\overline{\vartheta}_i}{\vartheta_j}$ and  ${\bf R}_{\varphi}\in\RR^{k_U\times k_U}$ be a coefficient matrix whose $ij$-th entry is $\dotp{\overline{\varphi}_i}{\varphi_j}$. Then, it follows that
\[
\hat\bbA^{-1} = \overline{\bm\varphi} ({\bf R}_{\varphi}^\top)^{-1}\begin{bmatrix}\hat{\alpha}_{c,1} &&\\
&\ddots&\\
&&\hat{\alpha}_{c,k_U}\end{bmatrix}^{-1}{\bf R}_{\vartheta}^{-1}\overline{{\bm\vartheta}}^\top,\quad
\hat\bbB' = \overline{\bm\vartheta} {\bf R}_{\vartheta}
\begin{bmatrix}\hat{\beta}_{y,c,1} &&\\
&\ddots&\\
&&\hat{\beta}_{y,c,k_U}\end{bmatrix}
{\bf R}_{\varphi}\overline{{\bm\varphi}}^\top.
\]
Hence, we can obtain
\begin{equation}\label{eq:A_invB}
\hat\bbA^{-1}\hat{\bbB}' = \overline{{\bm\varphi}}\underbrace{({\bf R}_\varphi^{\top})^{-1}\begin{bmatrix}\hat{\alpha}_{c,1} &&\\
&\ddots&\\
&&\hat{\alpha}_{c,k_U}\end{bmatrix}^{-1}
\begin{bmatrix}\hat{\beta}_{y,c,1} &&\\
&\ddots&\\
&&\hat{\beta}_{y,c,k_U}\end{bmatrix}
{\bf R}_{\varphi}}_{\bf M}\overline{{\bm\varphi}}^\top,
\end{equation}
where the eigenfunctions of $\hat{\bbA}^{-1}\hat{\bbB}'$ are obtained by first computing the eigenvectors of ${\bf M}$, denoted as $\hat{\bm\eta}_1,\ldots,\hat{\bm\eta}_{k_U}$ and then projecting them to the basis functions $\{\overline{\varphi}_1,\ldots,\overline{\varphi}_{k_U}\}$. That is, the $j$-th eigenfunction of $\hat{\bbA}^{-1}\hat{\bbB}'$ is $\sum_{i=1}^{k_U}\hat\eta_{j,i}\overline\varphi_i(w)$.

\emph{ Step 4 of Algorithm~\ref{algo:continuous:pw_u}}. Let $\hat{\bf D}=
\begin{bmatrix}\hat{\bf d}_1\cdots\hat{\bf d}_{k_U}\end{bmatrix}
=\begin{bmatrix}\hat{\bm\eta}_1\cdots\hat{\bm\eta}_{k_U}\end{bmatrix}^{-1}$, by proof of Theorem~\ref{thm:cont_q_y_x}, the estimate of ${p}(w\mid \tilde{U}=j)$ is $\sum_{i=1}^{k_U}\hat{d}_{j,i}\overline{\varphi}_i(w)/\norm{\sum_{i=1}^{k_U}\hat{d}_{j,i}\overline{\varphi}_i}_{L_1}$. 

\subsection{Implementation details of Algorithm~\ref{algo:continuous:qu_pu}}\label{section:details:qu_pu}
\begin{algorithm}[t]
\caption{Estimate $q(\tilde{U})/p(\tilde{U})$, details provided in Section~\ref{section:details:qu_pu}}\label{algo:continuous:qu_pu}
\begin{algorithmic}[1]
  \Require
  source $\mathcal{P} \!=\! \{(w_i, x_i)\}_{i=1}^n$;
  target $\mathcal{Q} \!=\! \{x_j\}_{j=1}^m$; $\{\hat{p}(W|U=1),\ldots,\hat{p}(W|U=k_U)\}$;
  \State Compute $\hat{p}(W\mid X)$ via LS-CDE~\citep{sugiyama2010conditional}
  \State Run K-means clustering to select $k_U$ centers: $x_{1},\ldots,x_{k_U}$
  \For{$x$ in $\{x_{1},\ldots,x_{k_U}\}$}
    \State $ \hat{p}(\bUt|X=x)$ is obtained by solving~\eqref{eq:lsq:pux} 
  \EndFor
  \State $[\hat{q}(\bUt)/\hat{p}(\bUt)]$ is obtained by solving~\eqref{eq:lsq:qupu}
\end{algorithmic}
\end{algorithm}

The first step of Algorithm~\ref{algo:continuous:qu_pu} is implemented by LS-CDE introduced in Appendix~\ref{appendix:intro:lscde} with the set of basis functions defined in~\eqref{eq:basisfunctions} and $\lambda=10^{-2}$. The second step is straightforward. Hence, we only discuss the implementation details of Step~4 in Algorithm~\ref{algo:continuous:qu_pu}. The identification result~\eqref{eq:identification:fu_x} suggests the construction of the following program
\begin{align}\label{eq:lsq:l2:pu_x}
\hat{p}(\tilde{\bf U}\mid X=x)=&\arg\min\quad\left\|\hat{p}(W\mid X=x)- \sum_{k=1}^{k_U} \hat{p}(W\mid \tilde{U}=i)p(\tilde{U}=k\mid X=x)\right\|_{L_2}^2\\
    &\text{subject to }\quad 0\leq p(\tilde{U}=i\mid x)\leq 1,\quad i=1,\ldots,k_U;\notag\\
    &\quad\quad\quad\quad\quad \sum_{i=1}^{k_U}p(\tilde{U}=i\mid x)=1.\notag
\end{align}
Define the design matrix ${\bf G}\in\RR^{k_U\times k_U}$:
\[
{\bf G} = \begin{bmatrix}
    \dotp{\hat{p}(W\mid\tilde{U}=1)}{\hat{p}(W\mid\tilde{U}=1)}&\cdots&
    \dotp{\hat{p}(W\mid\tilde{U}=1)}{\hat{p}(W\mid\tilde{U}=k_U)}\\
    \vdots&\ddots&\vdots\\
    \dotp{\hat{p}(W\mid\tilde{U}=k_U)}{\hat{p}(W\mid\tilde{U}=1)}&\cdots&
    \dotp{\hat{p}(W\mid\tilde{U}=k_U)}{\hat{p}(W\mid\tilde{U}=k_U)}    
    \end{bmatrix}.
\]
Given $x\in\Domain{X}$, since $p(\tilde{U}\mid x)$ is discrete random variable with $k_U$ states, we can reformulate~\eqref{eq:lsq:l2:pu_x} as
\begin{align}\label{eq:lsq:pux}
    \hat{p}(\tilde{\Ub}\mid x) =& \arg\min \left\|
    \begin{bmatrix}
    \dotp{\hat{p}(W\mid x)}{\hat{p}(W\mid\tilde{U}=1)}\\
    \vdots\\
    \dotp{\hat{p}(W\mid x)}{\hat{p}(W\mid\tilde{U}=k_U)}
    \end{bmatrix}
    -
    {\bf G}\begin{bmatrix}
    p(\tilde{U}=1\mid x)\\
    \vdots\\
    p(\tilde{U}=k_U\mid x)
    \end{bmatrix}
    \right\|_F^2,\\
    &\text{subject to }\quad 0\leq p(\tilde{U}=i\mid x)\leq 1,\quad i=1,\ldots,k_U;\notag\\
    &\quad\quad\quad\quad\quad \sum_{i=1}^{k_U}p(\tilde{U}=i\mid x)=1,\notag
\end{align}
which is a constrained least-squares problem and can be optimized efficiently by sequential least-squares programming. 

Finally, to compute the vector $q(\tilde{\Ub})/p(\tilde{\bf U})$, we need to estimate the marginal density $p(x)$ and $q(x)$. This can be implemented through a similar approach as introduced in $\emph{Step 1}$ in Appendix~\ref{appendix:step:continuous:pw_u}. We briefly introduce the procedure to estimate $p(x)$; The estimation procedure of $q(x)$ follows similarly. Consider the subspace spanned by Gaussian kernel basis functions of length-scale  $1$,
$$
\bigg\{\vartheta_\ell(x):\vartheta_\ell(x)=\exp\rbr{-\frac{\|x-\bar{x}_\ell\|^2}{2}},\ell=1,\ldots,k_U\bigg\}.
$$
We assume that the distribution is of the form $p(x)={\bm\alpha}^\top{\bm\vartheta}(x)$ with ${\bm\vartheta}(x)=\begin{bmatrix}
\vartheta_1(x)\cdots\vartheta_{k_U}(x)
\end{bmatrix}$. Hence, it follows that ${\bm\alpha}$ minimizes $\int({\bm\alpha}^\top{\bm\vartheta}(x)-p(x))^2dx$. Then, it is natural to formulate the empirical estimator as $\hat{p}(x)=\hat{\bm\alpha}^\top{\bm\vartheta}(x)$, where we can obtain $\hat{\bm\alpha}$ by solving the following problem:
\begin{align*}
\tilde{\bm\alpha} = \arg\min \frac{1}{2}{\bm\alpha}^\top{\bf H}_x{\bm\alpha} - \hat{\bf h}_x^\top{\bm\alpha} +\lambda{\bm\alpha}^\top{\bm\alpha},
\end{align*}
where ${\bf H}_x=\int{\bm\vartheta}(x){\bm\vartheta}(x)^\top dx$, and $\hat{\bf h}_x=\frac{1}{n}\sum_{i=1}^n{\bm\vartheta}(x_i)$. Then we set $\hat{\alpha}_i = \max(0, \tilde{\alpha}_i)$ for $i=1,\ldots,k_U$ to ensure the non-negativity of the distribution. 

After estimating $\hat{q}(x)$ and $\hat{p}(x)$, we construct the vector $[\hat{q}({\bf x})/\hat{p}({\bf x})]^\top=[\hat{q}(x_1)/\hat{p}(x_{1})\cdots\hat{q}(x_{k_U})/\hat{p}(x_{k_U})]^\top$ by querying $x_1,\ldots x_k$ from $\hat{q}(x)$ and $\hat{p}(x)$. Then we can obtain $\hat{q}(\bUt)/\hat{p}(\bUt)$ by solving the following constrained least-squares problem:
\begin{align}\label{eq:lsq:qupu}
[\hat{q}(\bUt)/\hat{p}(\bUt)] &= \arg\min\quad \left\|
\begin{bmatrix}
    \frac{\hat{q}(x_1)}{\hat{p}(x_{1})}\\\vdots\\\frac{\hat{q}(x_{k_U})}{\hat{p}(x_{k_U})}
\end{bmatrix}
-
\begin{bmatrix}
    p(\tilde{U}=1\mid x_1)& p(\tilde{U}=2\mid x_1)&\cdots
    &p(\tilde{U}=k\mid x_1)\\
    p(\tilde{U}=1\mid x_2)& p(\tilde{U}=2\mid x_2)&\cdots
    &p(\tilde{U}=k\mid x_2)\\
    \vdots&\ddots&&\vdots\\
    p(\tilde{U}=1\mid x_k)& p(\tilde{U}=2\mid x_k)&\cdots
    &p(\tilde{U}=k\mid x_k)
\end{bmatrix}
[{q}(\bUt)/{p}(\bUt)]^\top
\right\|_F^2\\
&\quad\;\text{subject to}\quad [{q}(\bUt)/{p}(\bUt)]_{i}\geq 0,\quad i=1,\ldots,k_U.\notag
\end{align}

\subsection{Implementation details of Algorithm~\ref{algo:continuous:pwy_x}}\label{appendix:details:pwy_x}

\begin{algorithm}[t]
\caption{Estimate $p(Y\mid \xnew, \tilde{U})$, details provided in Section~\ref{appendix:details:pwy_x}}\label{algo:continuous:pwy_x}
\begin{algorithmic}[1]
  \Require
  source $\mathcal{P} \!=\! \{(w_i,y_i ,x_i)\}_{i=1}^n$, $\hat{p}(\bUt\mid\xnew)$
  \State Compute $\hat{p}(Y\mid X)$ using MLP
  \For{$y=1,\ldots,k_Y$}
  \State Estimate ${p}(W\mid X, y)$ via LS-CDE~\citep{sugiyama2010conditional}
  \State Compute $\hat{p}(W,y\mid \xnew)=\hat{p}(W\mid X, y)\hat{p}(y\mid\xnew)$
  \EndFor
  \State Compute $\hat{p}({\bf Y}\mid \tilde{\Ub}, \xnew)$ by solving~\eqref{eq:lsq:py_xu}.
\end{algorithmic}
\end{algorithm}

\emph{}

In this section, we introduce the implementation details of Algorithm~\ref{algo:continuous:pwy_x}. First, we learn the distribution of ${p}(y|x)$ by fitting a Multi-Layer Perceptron classifier (MLP) with `ReLU' activation function attached at the output of the hidden layers. Given a fixed $y$, we estimate $p(W\mid X,y)$ by first constructing the subset of samples $\{x_i,w_i\}_{i\in \Ncal_y}$ such that $\Ncal_y=\{i\in[n]:y_i=y\}$. Then, $p(W\mid X,y)$ is estimated by fitting LS-CDE~\citep{sugiyama2010conditional} with $\{x_i,w_i\}_{i\in \Ncal_y}$ and the basis functions~\eqref{eq:basisfunctions}. 

To estimate $p(Y\mid \tilde{U},X)$, we first recall the relation of $p(W\mid \tilde{U})$, $p(\tilde{U}\mid X)$, $p(W,Y\mid X)$ and $p(Y\mid \tilde{U},X)$ defined in~\eqref{eq:identification:py_ux}. Computing the inverse of the matrix might lead to numerical instability in practice and hence we solves a constrained least-squares problem as an alternative. Given estimated $\{\hat\phi_i = \hat{p}(W\mid\tilde{U}=i)\}_{i=1}^{k_U}$, we run the Gram-Schmidt procedure to obtain $\hat{{\bphi}}_{\text{ortho}}=[\hat{{\phi}}_{\text{ortho},1},\ldots,\hat{{\phi}}_{\text{ortho},k_U}]$ and $\hat\Rb_{\phi}$. Let $\hat{\bf R}=\hat{\bf R}_\phi\diag(\hat{p}(\tilde{U}=1\mid\xnew),\ldots,\hat{p}(\tilde{U}=k_U\mid\xnew))$. Then, we estimate $p({\bf Y}\mid \tilde{\bf U},\xnew)\in\RR^{k_U\times k_Y}$  by solving the following constrained optimization problem
\begin{align}\label{eq:lsq:py_xu}
&\hat{p}({\bf Y}\mid \tilde{\bf U}, \xnew)=\notag\\
&\quad\arg\min\quad\left\|
\begin{bmatrix}
    \dotp{p(Y=1,W\mid \xnew)}{\hat{{\phi}}_{\text{ortho},1}}&\cdots&
    \dotp{p(Y=k_Y,W\mid \xnew)}{\hat{{\phi}}_{\text{ortho},1}}\\
    \vdots&\ddots&\vdots\\
    \dotp{p(Y=1,W\mid \xnew)}{\hat{{\phi}}_{\text{ortho},k_U}}&\cdots&
    \dotp{p(Y=k_Y,W\mid \xnew)}{\hat{{\phi}}_{\text{ortho},k_U}}
\end{bmatrix}
-({\bf I}\otimes_k \hat{\bf R} ){p}({\bf Y}\mid \tilde{\bf U}, \xnew)\right\|_F^2  \\
&\quad\text{subject to}\quad 0\leq p(Y=y\mid \tilde{U}=i,\xnew)\leq 1,\quad y=1,\ldots k_Y, i=1,\ldots,k_U;\notag\\
&\quad\quad\quad\quad\quad\;\sum_{y=1}^{k_Y}p(Y=y\mid \tilde{U}=i) = 1,\quad i = 1,\ldots,k_U\notag,
\end{align}
where $\otimes_k$ denotes the Kroneker product. This completes the procedure.
\section{Deep Latent Variable Model Setup}\label{app:vae}

As described in Section \ref{sect:estimation}, we approximate the joint distribution $p(X, Y, C, W, \tilde U)$ using a model based on the Wasserstein Auto-Encoder~\citep[WAE;][]{tolstikhin2018wasserstein}. The overall algorithm is broken down into five main steps as shown below:

\emph{High-Level Pseudo-code}
\begin{enumerate}
    \item Train the WAE.
    \item Use the WAE's encoder to append $\tilde U$ to the source dataset: $\{(x_i,y_i,c_i,w_i)\}_{i=1}^n\to\{(x_i,y_i,c_i,w_i,\tilde{u}_i)\}_{i=1}^n$.
    \item Train $p(\tilde U\mid X)$ and $p(Y\mid X,\tilde U)$ using the dataset $\{(x_i,y_i,c_i,w_i,\tilde{u}_i)\}_{i=1}^n$.
    \item Estimate the likelihood ratios $q(\tilde U)/p(\tilde U)$ using the confusion matrix approach of \citep{lipton2018detecting} and $p(\tilde U\mid X)$.
    \item Predict $q(Y\mid X)$ using  (\ref{eq:goal}).
\end{enumerate}
We describe, next, how this approach works in detail. 

\subsection{Training the WAE}\label{app:vae_1}
First, we approximate the latent variable $\tilde U$. For that, we construct a variant of WAE, in which the assumptions of the graph in  Figure~\ref{fig:model}(c) are imposed. Specifically, while the encoder $p(\tilde{U}\mid X,C,Y,W)$ is an MLP $(X, Y, C, W)\to \tilde U$ with parameters $\phi$, the decoder has the structure: $\tilde U\to X$, $\tilde U\to W$, $(\tilde U, X)\to C$, and $(\tilde U, C)\to Y$, where each arrow is a separate MLP model with its own parameters, leading to the factorization:
\begin{align*}
    p(\mathcal V,\tilde{U}) = p(Y\mid C,\tilde{U})\,p(C\mid X,\tilde{U})\,p(X\mid \tilde{U})\,p(W\mid \tilde{U})\,p(\tilde{U}),
\end{align*}
where $\mathcal{V}=(X, Y, C, W)$ as discussed in Section \ref{sect:estimation}. The WAE is trained to minimize the reconstruction loss and the KL-divergence between $p(\tilde U)$ and its prior $\overline{p}(\tilde U)$ as shown in (\ref{eq:vae_obj}), where $p$ is averaged over the entire batch. In our experiments, the reconstruction loss is the mean square error (MSE) for $X$, cross-entropy for $Y$ and $W$ (because both are one-hot encoded), and the binary cross-entropy for every concept in $C$ (because $C$ is multi-label). We set the number of latent categories $|\tilde U|$ in the WAE to 10. All MLPs follow the architecture described in Appendix \ref{sec:appendix:experiments}.

As $\tilde{U}$ is discrete, to allow training with the reparameterization trick, we model $p(\tilde{U}\mid X,C,Y,W)$ using a Gumbel-Softmax distribution \citep{jang2016categorical,maddison2016concrete}.
We set the prior $\overline{p}(\tilde{U})$ to be a uniform categorical distribution over the categories of $\tilde{U}$. 

\subsection{Append the latent category $\tilde U$}
Given a trained WAE model, we next generate joint samples $\{(x_i,c_i,y_i,w_i,\tilde{u}_i)\}_{i=1}^n$ using the encoder $ p(\tilde U \mid X,C,Y,W)$. Specifically, for every tuple in the training set $(x, y, c, w)$, we generate $\tilde u\sim  p(\tilde U \mid X=x,C=c,Y=y,W=w)$ and append $\tilde u$ to the tuple $(x, y, c, w)$.

\subsection{Training $p(\tilde U\mid X)$ and $p(Y\mid X, \tilde U)$}
Given the dataset $\{(x_i,c_i,y_i,w_i,\tilde{u}_i)\}_{i=1}^n$, we train a model $p(\tilde U\mid X)$ and another model $p(Y\mid X,\tilde U)$. In our experiments, both models are MLPs (of the same architecture specified in Appendix \ref{sec:appendix:experiments}), which are trained by minimizing the cross-entropy loss. After training, we calibrate on the separate hold-out dataset using temperature scaling \citep{guo2017calibration}.

\subsection{Likelihood Ratios}
Next, we employ the confusion matrix approach of \citep{lipton2018detecting} to estimate the likelihood ratios $q(\tilde U)/p(\tilde U)$ by applying it on the model $p(\tilde U|X)$. Specifically, since $p(\tilde U|X)$ is trained on source data, we calculate the confusion matrix on source. Then, we run the model on unlabeled $X$ from the target domain $q$ and calculate its mean predictions. After that, we use Proposition 2 in \citep{lipton2018detecting} to estimate the likelihood ratios.

\subsection{Inference}
Finally, during inference, we use the two models $p(\tilde U\mid X)$ and  $p(Y\mid X,\tilde U)$ trained in the third step and the likelihood ratios $q(\tilde U)/p(\tilde U)$ obtained in the forth step, and predict $q(Y|X)$ using (\ref{eq:goal}).

\end{document}